\documentclass{article}

 \PassOptionsToPackage{numbers}{natbib}

     \usepackage[final]{neurips_2020}

\usepackage[utf8]{inputenc} 
\usepackage[T1]{fontenc}    
\usepackage{hyperref}       
\usepackage{url}            
\usepackage{booktabs}       
\usepackage{amsfonts}       
\usepackage{nicefrac}       
\usepackage{microtype}      
\usepackage{amsmath,amssymb, amsthm}
\usepackage{centernot}
\usepackage{subfigure}
\usepackage{graphicx}
\usepackage{wrapfig}
\AtBeginDocument{%
  \addtolength\abovedisplayskip{-0.45\baselineskip}%
  \addtolength\belowdisplayskip{-0.45\baselineskip}%
  \addtolength\abovedisplayshortskip{-0.45\baselineskip}%
  \addtolength\belowdisplayshortskip{-0.45\baselineskip}%
}
\usepackage{amsmath,amsfonts,bm}

\newcommand{\tx}{\tilde{x}}
\newcommand{\tc}{\tilde{c}}
\newcommand{\tepsilon}{\tilde{\epsilon}}
\newcommand{\tdelta}{\tilde{\delta}}
\newcommand{\tX}{\tilde{X}}

\newcommand{\tf}{\tilde{f}}
\newcommand{\tD}{\widetilde{D}}
\newcommand{\tb}{\tilde{b}}

\newcommand{\td}{\tilde{d}}
\renewcommand{\tf}{\tilde{f}}

\newcommand{\cN}{\mathcal{N}}

\newcommand{\cF}{\mathcal{F}}
\newcommand{\cO}{\mathcal{O}}

\newcommand{\EE}{\mathbb{E}}
\newcommand{\RR}{\mathbb{R}}
\newcommand{\SSS}{\mathbb{S}}

\newcommand{\ind}[1]{\mathbb{I}_{\{#1\}}}

\newcommand{\dvl}{\nabla_{V_{\ell}}}
\newcommand{\dwl}{\nabla_{W_{\ell}}}
\newcommand{\diag}{{\rm diag}}

\newcommand{\mattwo}[4]{\left[\begin{matrix}#1 & #2 \\ #3 & #4\end{matrix}\right]}

\newcommand{\inner}[2]{\langle #1, #2 \rangle}
\newcommand{\biginner}[2]{\big\langle #1, #2 \big\rangle}

\DeclareMathOperator{\sign}{sign}
\DeclareMathOperator{\Tr}{Tr}

\newtheorem{definition}{Definition}
\newtheorem{prop}{Proposition}

\newtheorem{lem}{Lemma}
\newtheorem{theorem}{Theorem}
\newtheorem{rmk}{Remark}

\title{Why Do Deep Residual Networks Generalize Better than Deep Feedforward Networks? --- A Neural Tangent Kernel Perspective}

\author{%
  Kaixuan Huang\thanks{Equal contribution.}\\
  Peking University\\
  \texttt{hackyhuang@pku.edu.cn} \\
  \And 
  Yuqing Wang\footnotemark[1] \\
  Georgia Institute of Technology\\
  \texttt{ywang3398@gatech.edu  } \\ 
  \AND
  Molei Tao \\
  Georgia Institute of Technology\\
  \texttt{mtao@gatech.edu  } \\
  \And
  Tuo Zhao \\
  Georgia Institute of Technology\\
  \texttt{tourzhao@gatech.edu  } \\
}

\begin{document}

\maketitle

\begin{abstract}
Deep residual networks (ResNets) have demonstrated better generalization performance than deep feedforward networks (FFNets). However, the theory behind such a phenomenon is still largely unknown. This paper studies this fundamental problem in deep learning from a so-called ``neural tangent kernel'' perspective. Specifically, we first show that under proper conditions, as the width goes to infinity, training deep ResNets can be viewed as learning reproducing kernel functions with some kernel function. We then compare the kernel of deep ResNets with that of deep FFNets and discover that the class of functions induced by the kernel of FFNets is asymptotically not learnable, as the depth goes to infinity. In contrast, the class of functions induced by the kernel of ResNets does not exhibit such degeneracy. Our discovery partially justifies the advantages of deep ResNets over deep FFNets in generalization abilities. Numerical results are provided to support our claim.
\end{abstract}


\vspace{-0.2in}
\section{Introduction}
\vspace{-0.1in}


Deep Neural Networks (DNNs) have made significant progress in a variety of real-world applications, such as computer vision \citep{krizhevsky2012imagenet, goodfellow2014generative, Long_2015_CVPR}, speech recognition, natural language processing \citep{graves2013speech, bahdanau2014neural, young2018recent}, recommendation systems, etc. Among various network architectures, Residual Networks (ResNets, \citep{he2016deep}) are undoubtedly a breakthrough. Residual Networks are equipped with residual connections, which skip layers in the forward step. Similar ideas based on gating mechanisms are also adopted in Highway Networks \citep{srivastava2015training}, and further inspire many follow-up works such as Densely Connected Networks \citep{huang2017densely}.

Compared with conventional Feedforward Networks (FFNets), residual networks demonstrate surprising generalization abilities. Existing literature rarely considers deep feedforward networks with more than 30 layers. This is because many experimental results have suggested that very deep feedforward networks yield worse generalization performance than their shallow counterparts \citep{he2016deep}. In contrast, we can train residual networks with hundreds of layers, and achieve better generalization performance than that of feedforward networks. For example, ResNet-152 \citep{he2016deep}, achieving a $19.38\%$ top-1 error on the ImageNet data set, consists of 152 layers; ResNet-1001 \citep{he2016identity}, achieving a $4.92\%$ error on the CIFAR-10 data set, consists of 1000 layers. 

Despite the great success and popularity of the residual networks, the reason why they generalize so well is still largely unknown. There have been several lines of research attempting to demystify this phenomenon. One line of research focuses on empirical studies of residual networks, and provides intriguing observations. For example, \cite{veit2016residual} show that residual networks behave like an ensemble of weakly dependent networks of much smaller sizes, and meanwhile, they also show that the gradient vanishing issue is also significantly mitigated due to these smaller networks. \cite{balduzzi2017shattered} further provide a more refined elaboration on the gradient vanishing issue. They demonstrate that the gradient magnitude in residual networks only shows sublinear decay (with respect to the layer), which is much slower than the exponential decay of gradient magnitude in feedforward neural networks. \cite{li2018visualizing} propose a visualization approach for analyzing the landscape of neural networks, and further demonstrate that residual networks have smoother optimization landscape due to the skip-layer connections. 


Another line of research focuses on theoretical investigations of residual networks under {\it simplified network architectures}. A commonly adopted structure, which is a reformulation of FFNets, is 
\begin{equation}
    x_\ell = \phi (x_{\ell-1} + \alpha W_\ell x_{\ell-1}), \label{eq:res:type1}
\end{equation}
where $\ell$ is the number of layers and the skip-connection only bypasses the weight matrix $W_\ell$ at each layer \citep{allen2018convergence, zhang2019convergence, hardt2016identity,li2017convergence, liu2019towards}. Specifically, \cite{hardt2016identity} study the optimization landscape with linear activation; \cite{li2017convergence} study using Stochastic Gradient Descent (SGD) to train a two-layer ResNet. \cite{liu2019towards} study using Gradient Descent (GD) to train a two-layer non-overlapping residual network. \cite{allen2018convergence, zhang2019convergence} both take the perturbation analysis approach to show convergence of such ResNets. A more realistic structure is 
\begin{equation}
    x_\ell = x_{\ell-1} + \phi (\alpha W_\ell x_{\ell-1}), \label{eq:res:type2}
\end{equation}
where the skip-connection bypasses the activation function \citep{du2018gradient_2, frei2019algorithm}. \cite{frei2019algorithm} only consider separable setting and take the perturbation analysis to show the convergence and generalization property of such ResNet. These results, however, are only loosely related to the generalization abilities of residual networks, and often considered to be overoptimistic, due to the oversimplified assumptions.

Some more recent works provide a new theoretical framework for analyzing {\it overparameterized} neural networks \citep{jacot2018neural,arora2019exact, arora2019fine, allen2019can, allen2018convergence, allen2018learning, li2018learning, zou2018stochastic}. They focus on connecting two- or three-layer overparameterized (sufficiently wide) neural networks to {\it reproducing kernel Hilbert spaces}. Specifically, they show that under proper conditions, the weight matrices of a well trained overparameterized neural network (achieving any given small training error) are actually very close to their initialization. Accordingly, the training process can be described as searching within some class of reproducing kernel functions, where the associated kernel is called the ``{\it neural tangent kernel}'' ({\it NTK}, \citep{jacot2018neural}) and only depends on the initialization of the weights. Accordingly, the generalization properties of the overparameterized neural network are equivalent to those of the associated NTK function class. Based on such a framework, \cite{du2018gradient_2} derived the NTK of the ResNet \eqref{eq:res:type2} when only the last layer is trained, and proved the convergence of such ResNet. However, they did not provide an explicit formula for the NTK when all layers are trained, which is required for characterizing the generalization property of ResNets.


To better understand the generalization abilities of deep feedforward and residual networks, we propose to investigate the NTKs associated with these networks when all but the last layers are trained, and consider the case when both widths and depths go to infinity\footnote{More precisely, our analysis considers the regime, where the widths go to infinity first, and then the depths go to infinity. See more details in Section~\ref{sec:limitingntk}.}.
For the structure of ResNets, we adopt \eqref{eq:res:type2} only with a slight modification, since it captures the essence of the skip-connection; see Section~\ref{sec:back}
\begin{align}
 x_\ell = x_{\ell-1} + \alpha \sqrt{\frac{1}{m}} V_\ell \sigma_0 \Big( \sqrt{\frac{2}{m}} W_\ell x_{\ell-1} \Big).
\end{align}
Specifically, we prove that similar to what has been shown for feedforward networks \citep{jacot2018neural}, as the width of deep residual networks increases to infinity, training residual networks can also be viewed as learning reproducing kernel functions with some NTK. However, such an NTK associated with the residual networks exhibits a very different behavior from that of feedforward networks.

To demonstrate such a difference, we further consider the regime, where the depths of both feedforward and residual networks are allowed to increase to infinity. Accordingly, both NTKs associated with deep feedforward and residual networks converge to their limiting forms sublinearly (in terms of the depth). For notational simplicity, we refer to the limiting form of the NTKs as the limiting NTK. Besides asymptotic analysis, we also provide nonasymptotic bounds, which demonstrate equivalence between limiting NTKs and neural networks with sufficient depth and width.

When comparing their limiting NTKs, we find that the class of functions induced by the limiting NTKs associated with deep feedforward networks is essentially not learnable. Such a class of functions is sufficient to overfit training data. Given any finite sample size, however, the learned function cannot generalize. In contrast, the class of functions induced by the limiting NTKs associated with deep residual networks does not exhibit such degeneracy. Our discovery partially justifies the advantages of deep residual networks over deep feedforward networks in terms of generalization abilities. Numerical results are provided to support our claim. 

Our work is closely related to \cite{daniely2016toward}. They also investigate the so-called ``Gaussian Process'' kernel induced by feedforward networks under the regime where the depth is allowed to increase to infinity. However, their studied neural networks are essentially some specific implementations of the reproducing kernels using random features, since the training process only updates the last layer of the neural networks, and keeps other layers unchanged. In contrast, we assume the training process updates all layers except for the last layer.

\noindent {\bf Notations}: We use $\sigma_0(z)=\max(0,z)$ to denote the ReLU activation function in neural networks. We use $\sigma(z)$ to denote the normalized ReLU function $\sigma(z) = \sqrt{2}\max(0,z)$. The derivative
\footnote{Although the ReLU function $\sigma_0$ is not differentiable at $0$, we call $\sigma_0'$ derivative for notational convenience.} 
of ReLU function (step function) is $\sigma_0'(z) = \ind{z\geq 0}$. Then $\sigma'(z) = \sqrt{2}\ind{z\geq 0}$ is the normalized step function. We use $D$ to denote the input dimension and $\SSS^{D-1}$ to denote the unit sphere in $\RR^D$. We use $m$ to denote the network width (the number of neurons at each layer) and $L$ to denote the depth. Let $\mathcal{M}^2_+$ be the set of all $2\times 2$ positive semi-definite matrices. We use $\mathcal{F}$ to denote the set of all symmetric and positive semi-definite functions from $\RR^{D}\times\RR^{D}$ to $\RR$. We use $\|\cdot\|_{\max}$ to denote the entry-wise $\ell_\infty$ norm for matrices and use $\|\cdot\|$ to denote the $\ell_2$ norm for vectors and the spectral norm for matrices. We use $\mathrm{diag}(\cdot)$ to denote the diagonal matrix. We use $I_n$ to denote the $n\times n$ identity matrix. We use $x$ and $\tx$ to denote a pair of inputs. We use $x_\ell$ and $\tx_\ell$ to denote the output of the $\ell$-th layer of a network for the input $x$ and $\tx$, respectively. We use $f$ and $\tf$ to denote the final output of the network for $x$ and $\tx$, respectively. We use $\nabla_\theta f = \nabla_\theta f_\theta(x)$ to denote the derivative of parametrized model $f_\theta$ w.r.t. $\theta$ at the input $x$, and  $\nabla_\theta \tf$ to denote the counterpart at the input $\tx$. 

\vspace{-0.15in}
\section{Background}\label{sec:back}
\vspace{-0.1in}

For self-containedness, we first briefly review feedforward networks, residual networks and dual kernels associated with neural networks.


\vspace{-0.05in}

\noindent {\bf Feedforward Networks}. We define an $L$-layer feedforward network (FFNet) $f(x)$ with ReLU activation in a recursive manner,
\vskip -0.2in
\begin{align}
\label{eq:ffnet_def}
   x_0 = x; \ 
    x_\ell = \sqrt{\frac{2}{m}}\sigma_0(W_\ell x_{\ell-1}),\   \ell=1,\cdots,L;\ 
    f(x) = v^\top x_L,
\end{align}
\vskip -0.05in
where $W_1 \in \RR^{m \times D}$ and $W_2,\cdots,W_L \in \RR^{m\times m}$ are weight matrices, and $v \in \RR^{m}$ is the output weight vector. For simplicity, we only consider feedforward networks with scalar outputs.

\vspace{-0.05in}

\noindent {\bf Residual Networks}. We define an $L$-layer residual network (ResNet) $f(x)$ in a recursive manner, 
\vskip -0.2in
\begin{align}
\label{eq:resnet_def}
   x_0 = \sqrt{\frac{1}{m}} Ax;\ 
    x_\ell = x_{\ell-1} + \alpha \sqrt{\frac{1}{m}} V_\ell \sigma_0 \Big( \sqrt{\frac{2}{m}} W_\ell x_{\ell-1} \Big),~\ell=1,\cdots,L;\ 
    f(x) &= v^\top x_L,
\end{align}
\vskip -0.05in
where $W_\ell, V_\ell \in \RR^{m\times m}$ for $\ell=1,\cdots,L$, $A \in \RR^{m \times D}$, $v \in \RR^{m}$, and $\alpha = L^{-\gamma}$ is the scaling factor of the bottleneck layers. The scaling factor $\alpha$ is necessary for controlling the norm of $x_l$.

The network architecture in \eqref{eq:resnet_def} is similar to the ``pre-activation" shortcuts in \cite{he2016identity}, except that each bottleneck layer only contains one activation - between $W_\ell$ and $V_\ell$. We remove the activation of the input due to some technical issues (See more details in Section~\ref{sec:ntks}).

\vspace{-0.05in}

\noindent {\bf Dual and Normalized Kernels}. The dual kernel technique was first proposed in \cite{NIPS2009_3628} and motivated several follow-up works such as \cite{daniely2016toward, mairal2014convolutional}. Here we adopt the description in \cite{daniely2016toward}.  We use $K$ to denote a kernel function on the input space $\RR^D$, i.e., $K:\RR^D\times \RR^D\rightarrow \RR$. We denote
\vskip -0.1in
$$\Sigma(x,\tx)=\begin{pmatrix}
     K(x,x) & K(x,\tx) \\
        K(\tx,x) & K(\tx,\tx)
\end{pmatrix}~\textrm{and}~N_\rho=\begin{pmatrix}
     1 & \rho \\
        \rho & 1
\end{pmatrix},$$ where $K\in \mathcal{F}$, $\rho\in\RR$. Given an activation function $\phi :\RR \to \RR$, its dual activation function $\hat{\phi}: [-1,1]\to[-1,1]$ is defined to be $
\hat{\phi}(\rho)=\EE_{(X,\tilde{X})\sim\mathcal{N}(0,N_\rho)} \phi(X)\phi(\tilde{X}).
$

We then define the dual kernel as follows.
\begin{definition} We say that $\Gamma_\phi(K): \RR^D\times\RR^D \to \RR$ is the dual kernel of $K$ with respect to the activation $\phi$, if we have
$
    \Gamma_\phi(K) (x,\tx) =  \EE_{(X,\tilde{X})\sim\mathcal{N}(0,\Sigma(x,\tx))} \phi(X)\phi(\tilde{X}).
$

\end{definition} 
Note that $\Gamma_\phi(K)$ is also positive semi-definite. We also define the normalized kernel.
\begin{definition}
We say that a kernel $K \in \cF$ is {\it normalized}, if $K(x,x) = 1$ for all $x \in \RR^D$. For a general kernel $K \in \cF$, we define its normalized kernel by $\overline{K}$ where  
$
    \overline{K} (x,\tx) = \frac{K(x,\tx)}{\sqrt{K(x,x)K(\tx,\tx)}}.
$
\end{definition}

For {\it normalized ReLU function} $\sigma(z) = \sqrt{2} \max(0,z)$, \cite{daniely2016toward} show
$
    \hat{\sigma}(\rho) = \frac{\sqrt{1-\rho^{2}}+\left(\pi-\cos ^{-1}(\rho)\right) \rho}{\pi}.
$
Since $\sigma(z)$ is positive homogeneous, we have
$
\Gamma_\sigma(K)(x,\tx) = \sqrt{K(x,x)K(\tx,\tx)} \  \hat{\sigma}(\overline{K}(x,\tx)).
$
For {\it derivative of normalized ReLU function} $\sigma'(z) = \sqrt{2} \ind{z\geq 0}$, \cite{daniely2016toward} show that 
$
 \widehat{\sigma'}(\rho) = \frac{\pi-\cos ^{-1}(\rho)}{\pi}.
$
Since $\sigma'(z)$ is zeroth-order positive homogeneous, we have
$
\Gamma_{\sigma'}(K)(x,\tx) =  \widehat{\sigma'}(\overline{K} (x,\tx)).
$ For more technical details of the dual kernel, we refer the readers to \cite{daniely2016toward}.


\vspace{-0.15in}
\section{Neural Tangent Kernels of Deep Networks}
\label{sec:ntks}
\vspace{-0.1in}

There are two approaches to connecting neural networks to kernels: one is {\it Gaussian Process Kernel} (GP Kernel); the other is {\it Neural Tangent Kernel} (NTK). GP Kernel corresponds to the regime where the first $L$ layers are fixed after random initialization, and only the last layer is trained. Therefore, the first $L$ layers are essentially random feature mapping \cite{rahimi2008random}. This is inconsistent with the practice, as the first $L$ layers should also be trained. In contrast, NTK corresponds to the regime where the first $L$ layers are also trained. For both GP Kernel and NTK, we consider the case when the width of the neural network goes to infinity. Due to space limit, we only provide some proof sketches for our theory, and all technical details are deferred to the appendix.

\vspace{-0.15in}
\subsection{Feedforward Networks}
\vspace{-0.1in}

We consider the Feedforward Network (FFNet) defined in \eqref{eq:ffnet_def}, where $W_1 \in \RR^{m \times D}$, $W_2,\cdots,W_L \in \RR^{m\times m}$ and $v \in \RR^{m}$ are all initialized as i.i.d. $\cN(0,1)$ variables.\footnote{In general, the weight matrices do not need to be square matrices, nor do they need to be of the same size.} Given such random initialization, the outputs converge to a Gaussian process, as the width goes to infinity \citep{lee2017deep,jacot2018neural}. Accordingly, the GP kernel is defined as follows.

\begin{prop}[\cite{daniely2016toward,jacot2018neural}]\label{GP-Kernel-FFNet}
The GP kernel of the $L$-layer FFNet defined in \eqref{eq:ffnet_def} is
\vskip -0.2in
\begin{align}
\label{eq:feednn_gpk}
    K_0(x,\tx) = x^\top \tx;\  
    K_{\ell}(x,\tx) = \Gamma_{\sigma}(K_{\ell-1})(x,\tx),\ \ell=1,\cdots,L.
    \end{align}
\end{prop}

\begin{theorem}[\cite{daniely2016toward}]
For the FFNet defined in \eqref{eq:ffnet_def}, there exists an absolute constant $C$, given the width $m\geq C\epsilon^{-2}L^2\log(8L/\delta),$ with probability at least $1-\delta$ over the randomness of the initialization, for input $x,\tx$ on the unit sphere, the inner product of the outputs of the $\ell$-th layer can be approximated by $K_\ell(x,\tx)$, i.e., 
$$
    |\inner{x_\ell}{\tx_\ell} - K_\ell(x,\tx)| \leq \epsilon, \text{ for all } \ell=1,\cdots,L.
$$
\end{theorem}
\vskip -0.03in
The next proposition shows the NTK of this FFNet. Unlike the GP kernel, the NTK corresponds to the case when $\theta = (W_1,\cdots,W_L)$ are trained.
\begin{prop}[\cite{jacot2018neural}]\label{infinite-NTK-FFNets}
The NTK of the FFNet can be derived in terms of the GP kernels as
\vskip -0.2in
\begin{align}
    \Omega_L (x,\tx) = \sum_{\ell=1}^L \Big[ K_{\ell-1}(x,\tx) \prod_{i=\ell}^L \Gamma_{\sigma'}(K_{i-1})(x,\tx) \Big]. \label{eq:FeedNTK_exp}
\end{align}
\end{prop}
\vskip -0.03in
Besides the asymptotic result, \cite{arora2019exact} further provide a nonasymptotic bound as follows.
\begin{theorem}[\cite{arora2019exact}]\label{finite-FFNets-NTK}
For the FFNet defined in \eqref{eq:ffnet_def}, when the width $m \geq CL^6\epsilon^{-4} \log(L/\delta)$, where $C$ is a constant, with probability at least $1-\delta$ over the initialization, for input $x,\tx$ on the unit sphere, the Neural Tangent Kernel can be approximated by $\Omega_L (x,\tx)$, i.e., 
$$
    \big\vert \langle\nabla_\theta f,\nabla_\theta \tf\rangle - \Omega_L (x,\tx) \big\vert \leq L\epsilon.
$$
\end{theorem}
\vskip -0.03in
\cite{arora2019exact} then showed that a sufficiently wide FFNet trained by gradient flow is close to the kernel regression predictor via its NTK. 

\begin{rmk}
For self-containedness, we directly adopt the results from existing literature in this subsection. For more technical details on gradient flow and kernel ridge regression, we refer the readers to \cite{daniely2016toward,jacot2018neural,arora2019exact}.
\end{rmk}

\vspace{-0.1in}
\subsection{Residual Networks}
\vspace{-0.05in}

We consider the Residual Network (ResNet) in \eqref{eq:resnet_def}, where all parameters ($A, v, W_1,\cdots,W_L,V_1,\cdots,V_L $) are independently initialized from the standard Gaussian distribution. For simplicity, we only train $\theta =(W_1,\cdots,W_L,V_1,\cdots,V_L)$, but not $A$ or $v$, and the NTK of the ResNet is computed accordingly. Note that our theory can be naturally generalized to the setting where all parameters including $A$ and $v$ are trained, but the analysis will be more involved.
Our next proposition derives the GP kernel of the ResNet.
\begin{prop}
\label{prop:GPKRN}
The GP kernel of the ResNet is 
\begin{align*}
    K_0 (x,\tx) = x^\top \tx;\ 
    K_\ell (x,\tx) = K_{\ell-1}(x,\tx) + \alpha^2 \Gamma_\sigma(K_{\ell-1})(x,\tx), 
\end{align*}
where $\ell=1,\cdots,L$, and $\alpha =  L^{-\gamma}$ for $0.5 \leq \gamma \leq 1$.
\end{prop}
Proposition \ref{prop:GPKRN} demonstrates that each layer of the ResNet recursively ``contributes'' to the kernel in an incremental manner, which is quite different from that of the FFNet (shown in Proposition \ref{GP-Kernel-FFNet}). Proposition \ref{prop:GPKRN} essentially provides a rigorous justification for the intuition discussed by \cite{garriga2018deep}. Besides the above asymptotic result, we also derive a nonasymptotic bound as follows.

\begin{theorem}\label{thm:gp}
For the ResNet defined in \eqref{eq:resnet_def}, given two inputs on the unit sphere $x,\tx \in \SSS^{D-1}$, $\epsilon < 0.5$, and 
$$
m \geq C\epsilon^{-2}L^{2-2\gamma}\log(36(L+1)/\delta),
$$
where $C$ is a constant and $0.5 \leq \gamma \leq 1$, with probability at least $1-\delta$ over the randomness of the initialization, for all layers $\ell=0,\cdots,L$ and $(x^{(1)},x^{(2)}) \in \{ (x,x), (x,\tx), (\tx,\tx) \}$, we have
$$
    | \langle x_\ell^{(1)}, x_\ell^{(2)} \rangle - K_\ell(x^{(1)},x^{(2)}) | \leq \epsilon,
$$
where $K_\ell$ is recursively defined in Proposition \ref{prop:GPKRN}. 
\end{theorem}

Theorem \ref{thm:gp} implies that sufficiently wide residual networks are mimicking the GP kernel under proper conditions. The proof can be found in Appendix~\ref{sec:proof_resgpk}. Next we present the NTK of the ResNet defined in \eqref{eq:resnet_def} in the following proposition.
\begin{prop}
\label{prop:rnntk}
The NTK of the ResNet is
$
\Omega_L (x,\tx) = \alpha^2 \sum_{\ell =1}^L \big[B_{\ell +1}(x,\tx) \Gamma_{\sigma}(K_{\ell -1})(x,\tx)  + K_{\ell -1}(x,\tx)B_{\ell +1}(x,\tx)\Gamma_{\sigma'}(K_{\ell -1})(x,\tx)\big],
$
where $K_\ell$'s are defined in Proposition \ref{prop:GPKRN}; $B_{L +1}(x,\tx) = 1$, and for $\ell=1,\cdots,L$, $B_\ell$'s are defined as
$$
    B_{\ell +1}(x,\tx) = B_{\ell +2}(x,\tx) + \alpha^2 B_{\ell +2}(x,\tx) \Gamma_{\sigma'}(K_\ell )(x,\tx).
$$
\end{prop}

Proposition \ref{prop:rnntk} implies that similar to what has been proved for the FFNet, the ResNet trained by gradient flow is also equivalent to the kernel regression predictor with some NTK. Note that Proposition \ref{prop:rnntk} is an asymptotic result. We defer the proof, as it can be straightforwardly derived from the nonasymptotic bound as follows.

\begin{theorem}\label{thm:ntk}
For the ResNet defined in \eqref{eq:resnet_def}, given two inputs on the unit sphere $x,\tx \in \SSS^{D-1}$, $\epsilon < 0.5$, and
$$
    m\geq  C\epsilon^{-4}L^{2-2\gamma}\big(\log(320(L^2+1)/\delta)+1\big),
$$
where $C$ is a constant, with probability at least $1-\delta$ over the randomness of the initialization, we have
$$
\big  \vert \biginner{\nabla_\theta f}{\nabla_\theta \tf}  - \Omega_L(x,\tx) \big \vert \leq 2L\alpha^2\epsilon,
$$
where $\alpha = L ^{-\gamma}$ with $\gamma \in [0.5,1]$, $ \Omega_L(x,\tx)$ is defined in Proposition~\ref{prop:rnntk}.
\end{theorem}

\begin{proof}[Proof Sketch of Proposition~\ref{prop:rnntk} and Theorem~\ref{thm:ntk}]

For simplicity, we use $\phi_W: \RR^m \to \RR^m$ to denote
$
    \phi_W(z) = \sqrt{\frac{2}{m}} \sigma_0(Wz).
$
Then its derivative w.r.t. $z$ is as follows,
$
    \phi_W' (z) = \sqrt{\frac{2}{m}}D(Wz)W,
$
where $D(Wz)$ is an operator defined as
$D(Wz) \equiv \diag(\sigma_0'(Wz))
 = \diag([\ind {W_{1,\cdot}z \geq 0}, \cdots, \ind{W_{m,\cdot}z \geq 0}]^\top).
$

For simplicity, we denote $D_\ell = D(W_\ell x_{\ell-1})$, where $\ell=1,2,\cdots,L$. Note that $D_\ell$ is essentially the activation pattern of the $\ell$-th bottleneck layer on the input $x$. We denote $\tD_\ell$ for $\tx$ in a similar fashion. Then we have
$
\frac{\partial x_{\ell}}{\partial x_{\ell-1}} = I_m + \alpha \sqrt{\frac{1}{m}} V_\ell \sqrt{\frac{2}{m}}D_\ell W_\ell  .
$
For $\ell=1,\cdots,L$, we denote $b_{\ell+1} = \nabla_{x_\ell}f$. Then we have
$
    b_{\ell+1} = \big(v^\top \frac{\partial x_L }{\partial x_{L -1}} \frac{\partial x_{L -1}}{\partial x_{L -2}}\cdots \frac{\partial x_{\ell+1}}{\partial x_\ell} \big)^\top.
$

Combining all above derivations, we have
$
\nabla_{V_\ell} f = \frac{\alpha}{\sqrt{m}} b_{\ell+1} \cdot (\phi_{W_\ell}(x_{\ell-1}))^\top,$ and $
\nabla_{W_\ell} f = \frac{\alpha}{\sqrt{m}} \sqrt{\frac{2}{m}} D_\ell V_\ell^\top b_{\ell+1} \cdot x_{\ell-1}^\top .
$
Then we can derive the kernel
$
\sum_{\ell =1}^L \inner{\nabla_{W_\ell } f}{\nabla_{W_\ell } \tf} +\sum_{\ell =1}^L \inner{\nabla_{V_\ell } f}{\nabla_{V_\ell } \tf},~\textrm{where}
$
$
\!\inner{\nabla_{V_\ell} f}{\nabla_{V_\ell} \tf}\! =\! \alpha^2 \underbrace{\frac{1}{m} \!\langle b_{\ell+1},\tb_{\ell+1}\rangle}_{T_{\ell,1}} \!\underbrace{\inner{\phi_{W_\ell}(x_{\ell -1})}{\phi_{W_\ell }(\tx_{\ell -1})}}_{T_{\ell,2}},\\
\!\inner{\nabla_{W_\ell } f}{\nabla_{W_\ell } \tf} \!= \!\alpha^2 \underbrace{\! \inner{x_{\ell -1}}{\tx_{\ell -1}}}_{T_{\ell,3}}\!\underbrace{\frac{2}{m^2} \tb_{\ell +1} ^\top V_\ell   \tD_\ell  D_\ell  V_\ell ^\top b_{\ell +1}}_{T_{\ell,4}}.
$
Note that the concentration of $T_{\ell,3}$ can be shown by Theorem \ref{thm:gp}. We then show the concentration of $T_{\ell,1}$, $T_{\ell,2}$ and $T_{\ell,4}$, respectively.

For simplicity, we define two matrices for each layer,
\begin{align*}
   \widehat{\Sigma}_\ell (x,\tx) = \mattwo{\inner{x_{\ell }}{x_\ell }}{\inner{x_{\ell }}{\tx_\ell }}{\inner{\tx_{\ell }}{x_\ell }}{\inner{\tx_{\ell }}{\tx_\ell }},\ 
   \Sigma_\ell (x,\tx) = \mattwo{K_\ell (x,x)}{K_\ell (x,\tx)}{K_\ell (\tx,x)}{K_\ell (\tx,\tx)}.
\end{align*}
We define $\psi_\sigma : \mathcal{M}^{2}_+ \to \RR$ as $
    \psi_\sigma(\Sigma) = \EE_{(X,\tX) \sim \cN(0, \Sigma)} \sigma(X)\sigma(\tX)$ and $\psi_{\sigma'} :\mathcal{M}^{2}_+ \to \RR$ as 
$
    \psi_{\sigma'}(\Sigma) = \EE_{(X,\tX) \sim \cN(0, \Sigma)} {\sigma'}(X)\sigma'(\tX).
$
Note $\Gamma_\sigma(K_{\ell-1})=\psi_\sigma(\Sigma_{\ell-1})$ and $\Gamma_{\sigma'}(K_{\ell-1})=\psi_{\sigma'}(\Sigma_{\ell-1})$.

The following lemmas are technical results and very involved. Please see Appendix~\ref{sec:proof_resntk} for details.
\begin{lem}
\label{lem:skt:T1}
Suppose that for $\ell=1,\cdots,L$,
\begin{align}
\|\widehat{\Sigma}_{\ell-1}(x,\tx)-\Sigma_{\ell-1}(x,\tx)\|_{\max} \leq c\epsilon^2,~~m\geq C_1\epsilon^{-2}L^{2-2\gamma} \big(\log(80L^2/\delta)+1 \big), \label{eqn:large-enough-m}
\end{align}
with probability at least $1-3\delta$, we have
$| T_{\ell,1}-B_{\ell+1}(x,\tx) | \leq c_1\epsilon$, for $\ell=1,\cdots,L$,
where $C_1$, $c_1$, and $c$ are constants.
\end{lem}
\begin{lem}
\label{lem:skt:T2}
Suppose \eqref{eqn:large-enough-m} holds for $\ell=1,\cdots,L$. With probability at least $1-\delta$, we have
$
| T_{\ell,2} - \Gamma_\sigma(K_{\ell-1})(x,\tx) | \leq c_2\epsilon,
$  for $\ell=1,\cdots,L$, 
where $C_2$ and $c_2$ are constants.
\end{lem}
\begin{lem}
\label{lem:skt:T4}
Suppose that \eqref{eqn:large-enough-m} holds for $\ell=1,\cdots,L$. With probability at least $1-3\delta$, we have $| T_{\ell,4} - B_{\ell+1}(x,\tx)\Gamma_{\sigma'}(K_{\ell-1})(x,\tx) | \leq c_3\epsilon
$,  for $\ell=1,\cdots,L$,
where $c_3$ is a constant.
\end{lem}
We remark: (1) Lemma \ref{lem:skt:T1} is proved by reverse induction; (2) Lemma~\ref{lem:skt:T2} exploits the concentration properties of $W_\ell$ and local Lipschitz properties of $\psi_\sigma$; (3) We prove Lemma~\ref{lem:skt:T4} and Lemma~\ref{lem:skt:T1} simultaneously with the H\"{o}lder continuity of $\psi_{\sigma'}$. Combining all results above, we complete Theorem~\ref{thm:ntk}. Moreover, taking $m\rightarrow \infty$, we have Proposition~\ref{prop:rnntk}.
\end{proof}


\vspace{-0.15in}
\section{Deep Feedforward v.s. Residual Networks}
\label{sec:limitingntk}
\vspace{-0.1in}

To compare the NTKs associated with deep FFNets and ResNets, we consider proper normalization, which avoids the kernel function blowing up or vanishing as the depth $L$ goes to infinity.

\vspace{-0.1in}
\subsection{The Limiting NTK of the Feedforward Networks}
\vspace{-0.05in}

Recall that the NTK of the $L$-layer FFNet defined in \eqref{eq:ffnet_def} is
$
    \Omega_L (x,\tx) = \sum_{\ell =1}^L \big[ K_{\ell -1}(x,\tx) \cdot \prod_{i=\ell} ^L \Gamma_{\sigma'}(K_{i-1})(x,\tx) \big].
$
One can check that $\Omega_L (x,x)=L$ for all $x\in\SSS^{D-1}$. To avoid $\Omega_L (x,x)\rightarrow\infty$, as $L\rightarrow \infty$. We consider a normalized version as 
$$    \overline{\Omega}_L (x,\tx) =\frac{1}{L }\Omega_L (x,\tx).
$$ We characterize the impact of the depth $L$ on the NTK in the following theorem.
\begin{theorem}
\label{thm:ffntkinf}
For the NTK of the FFNet, as $L\rightarrow\infty$, given $x,\tx \in \SSS^{D-1}$ and $|1-x^\top\tx|\geq\delta>0,$ where $\delta$ is a constant and does not scale with 
$L$, we have
$$\Big\vert \overline{\Omega}_L(x,\tilde{x})-1/4 \Big\vert = \mathcal{O}\Big(\frac{\mathrm{polylog}(L)}{L}\Big),
$$ When $x = \tx$, we have $\overline{\Omega}_L (x,\tx)=1, \forall L$.
\end{theorem}
 \begin{proof}[Proof Sketch of Theorem \ref{thm:ffntkinf} ]
 The main challenge comes from the sophisticated recursion of the kernel. To handle the recursion, we employ the following bound.
 \begin{lem}
 \label{lem:skt:bounds}
 When $L$ is large enough, we have
 \begin{align*}
 \cos \left(\pi  \left(1-\left(\frac{n}{n+1}\right)^{3+\frac{\log(L)^2}{L}}\right)\right)\leq K_n(x,\tx)\leq \cos \left(\pi  \left(1-\left(\frac{n+\log(L)^p}{n+\log(L)^p+1}\right)^{3-\frac{\log(L)^2}{L}}\right)\right),
 \end{align*}
 where $p$ is a positive constant depending on $\delta$.
 \end{lem}
 By Lemma \ref{lem:skt:bounds}, we can further bound $\prod_{i=\ell}^{L}\Gamma_{\sigma'}(K_{i-1}(x,\tx))$ by
 \begin{align}
 \Big(\frac{\ell-1}{L}\Big)^{3+\frac{\log(L)^2}{L}}\leq\prod_{i=\ell}^{L}\Gamma_{\sigma'}(K_{i-1}(x,\tx))\leq\Big(\frac{\ell+\log(L)^p-1}{L+\log(L)^p}\Big)^{3-\frac{\log(L)^2}{L}}
 \end{align}
 Hence we can measure the rate of convergence. The detailed proof is the following.
 \end{proof}


As can be seen from Theorem \ref{thm:ffntkinf}, the NTK of the FFNet converges to a limiting form, i.e.,
\begin{align*}
   \overline{\Omega}_{\infty}(x,\tx)= \lim_{L  \to \infty} \overline{\Omega}_L(x,\tx) = \left\{
    \begin{array}{cc}
        1/4, & x \ne \tx \\
        1, &  x=\tx
    \end{array}.
    \right.
\end{align*}
For simplicity, we refer to $\overline{\Omega}_{\infty}$ as the limiting NTK of the FFNets.

The limiting NTK of the FFNets is actually a non-informative kernel. For example, we consider a kernel regression problem with $n$ independent observations $\{(x_i,y_i)\}_{i=1}^n$, where $x_i\in\RR^D$ is the feature vector, and $y_i\in\RR$ is the response. Without loss of generality, we assume that the training samples have been properly processed such that $x_i\ne x_j$ for $i\ne j$, and $\sum_{i=1}^ny_i=0.$ By the Representer theorem \citep{friedman2001elements}, we know that the kernel regression function can be represented by
$
f(\cdot) = \sum_{i=1}^n\beta_i \overline{\Omega}_{\infty}(x_i,\cdot).
$
We then minimize the regularized empirical risk as follows.
\begin{align}\label{kernel-regression-erm}
\hat{\beta}=\min_\beta \|y- \widetilde{\Omega}\beta\|^2+\lambda\beta^\top  \widetilde{\Omega}\beta,
\end{align} 
where $\beta=(\beta_1,...\beta_n)^\top\in\RR^{n}$, $y=(y_1,...,y_n)^\top\in \RR^n$, $\widetilde{\Omega}\in\RR^{n\times n}$ with $\widetilde{\Omega}_{ij}=\overline{\Omega}_{\infty}(x_i,x_j)$, and $\lambda$ is the regularization parameter and usually very small for large $n$. One can check that \eqref{kernel-regression-erm} admits a closed form solution $\hat{\beta}=(\widetilde{\Omega}+\lambda I_n)^{-1}y.$ Note that we have 
$ \tilde{\Omega}+\lambda I_n =1/4J_n + (\lambda+3/4)  I_n$, which is the sum of a diagonal matrix and a rank-one matrix and $J_n$ is $n\times n$ all-ones matrix. By Sherman -- Morrison formula $$(A+uv^\top)^{-1}=A^{-1}-\frac{A^{-1}uv^\top A^{-1}}{1+v^\top A^{-1}u},\  \text{we have}
\ \hat{\beta}=\frac{1}{\lambda +3/4}\Big(I_n-\frac{1}{n+4\lambda +3} J_n\Big)y.$$
Then we further have
$
f(x_j) = \sum_{i=1}^n\hat{\beta}_i \overline{\Omega}_{\infty}(x_i,x_j)=\frac{3}{4\lambda +3}y_j.
$

As can be seen, for sufficiently large $n$ and sufficiently small $\lambda$, we have $f(x_j)\approx y_j$, which means that we can fit the training data well. However, for an unseen data point $x^*$, where $x^*\neq x_1,...,x_n$, the regression function $f$ always gives an output $0$, i.e.,
$$
f(x^*) = \sum_{i=1}^n\hat{\beta}_i \overline{\Omega}_{\infty}(x_i,x^*) = \frac{1}{4}\sum_{i=1}^n\hat{\beta}_i =0.
$$
This indicates that the function class induced by the limiting NTK of the FFNets $\overline{\Omega}_{\infty}$ is not learnable.

\vspace{-0.1in}
\subsection{The Limiting NTK of the Residual Networks}
\vspace{-0.05in}

Recall that the infinite-width NTK of the $L$-layer ResNet is
$$
\Omega_L (x,\tx) = \alpha^2 \sum_{\ell =1}^L \Big[B_{\ell +1}(x,\tx) \Gamma_{\sigma}(K_{\ell -1})(x,\tx) + K_{\ell -1}(x,\tx)B_{\ell +1}(x,\tx)\Gamma_{\sigma'}(K_{\ell -1})(x,\tx)\Big],
$$
where $B_{L +1}(x,\tx)=1$ and for $\ell=1,..,L-1$, $B_{\ell +1}(x,\tx) = \prod_{i=\ell} ^{L -1} (1+\alpha^2 \Gamma_{\sigma'}(K_i)(x,\tx)).$
One can check that for $x \in \SSS^{D-1}$, $\Omega_L (x,x) = 2L \alpha^2  (1+\alpha^2)^{L -1}.$ 

Different from the NTK of the FFNet, $\Omega_L (x,x)\to 0$ as $L\to\infty$. Therefore, we also consider the normalized NTK for the ResNet to prevent the kernel from vanishing. Specifically, the normalized NTK of the ResNet on $\SSS^{D-1} \times \SSS^{D-1}, \ \overline{\Omega}_L (x,\tx),$ is defined as follows,
\begin{align}
\label{limiting-resnet-ntk}
   \frac{1/(2L)}{(1+\alpha^2)^{L -1}} \sum_{\ell =1}^L \Big[B_{\ell +1}(x,\tx) \Gamma_{\sigma}(K_{\ell -1})(x,\tx) + K_{\ell -1}(x,\tx)B_{\ell +1}(x,\tx)\Gamma_{\sigma'}(K_{\ell -1})(x,\tx)\Big].
    \end{align}
We then analyze the limiting NTK of the ResNets. Recall that $\alpha = L^{-\gamma}$. Our next theorem only considers $\gamma=1$, i.e., $\alpha=1/L$. 
\begin{theorem}
\label{thm:resntkinf}
For the NTK of the ResNet, as $L\to \infty$, given $\alpha=\frac{1}{L}$ and $x,\tx \in \SSS^{D-1}$ such that $|1-x^\top\tx|\geq\delta>0,$ where $\delta$ is a constant and does not scale with 
$L$, we have
$$
\left\vert \overline{\Omega}_L (x,\tilde{x})-\overline{\Omega}_1 (x,\tilde{x}) \right\vert  = \mathcal{O}\left(1/L\right),
$$
where $\overline{\Omega}_1(x,\tilde{x})=\frac{1}{2}\left(\hat \sigma(x^\top\tx)+x^\top\tx \cdot \hat{\sigma'}(x^\top\tx)\right)$.
\end{theorem}

 \begin{proof}[Proof Sketch of Theorem~\ref{thm:resntkinf}] The main technical challenge here is also handling the recursion. Specifically, we denote $K_{\ell,L}$ to be the $\ell$-th layer of the GP kernel when the depth is $L$, which is originally denoted by $K_\ell(x,\tx)$. Let $S_0=K_0(x,\tx)$ and $S_{\ell,L}=\frac{K_{\ell,L}}{(1+\alpha^2)^\ell}=\frac{K_{\ell,L}}{(1+1/L^2)^\ell}.$ We have $\Gamma_\sigma(K_{\ell,L})=(1+\alpha^2)^\ell\hat{\sigma}(S_{\ell,L})$ and $\Gamma_{\sigma'}(K_{\ell,L})=\widehat{\sigma'}(S_{\ell,L})$. We rewrite the recursion of $K_{\ell,L}$ as
 $
 S_{\ell,L}=\frac{S_{\ell-1,L}+\alpha^2\hat{\sigma}(S_{\ell-1,L})}{(1+\alpha^2)}\geq S_{\ell-1,L},
 $ which eases the technical difficulty. However, the proof is still highly involved, and more details can be found in Appendix~\ref{proof:limitingrnntk}.
 \end{proof}

Note that we do not consider $\gamma=0.5$ for technical concerns, as $\overline{\Omega}_L (x,\tx)$ in \eqref{limiting-resnet-ntk} becomes very complicated to compute, as $L\rightarrow \infty$. Also we find that considering $\gamma=1$ is sufficient to provide us new theoretical insights on ResNets (See more details in Section~\ref{sec:exp}).

\begin{figure*}[htb!]
\vspace{-0.2in}
\begin{center}
    \subfigure[\it FFNets]{
        \includegraphics[width=0.27\textwidth]{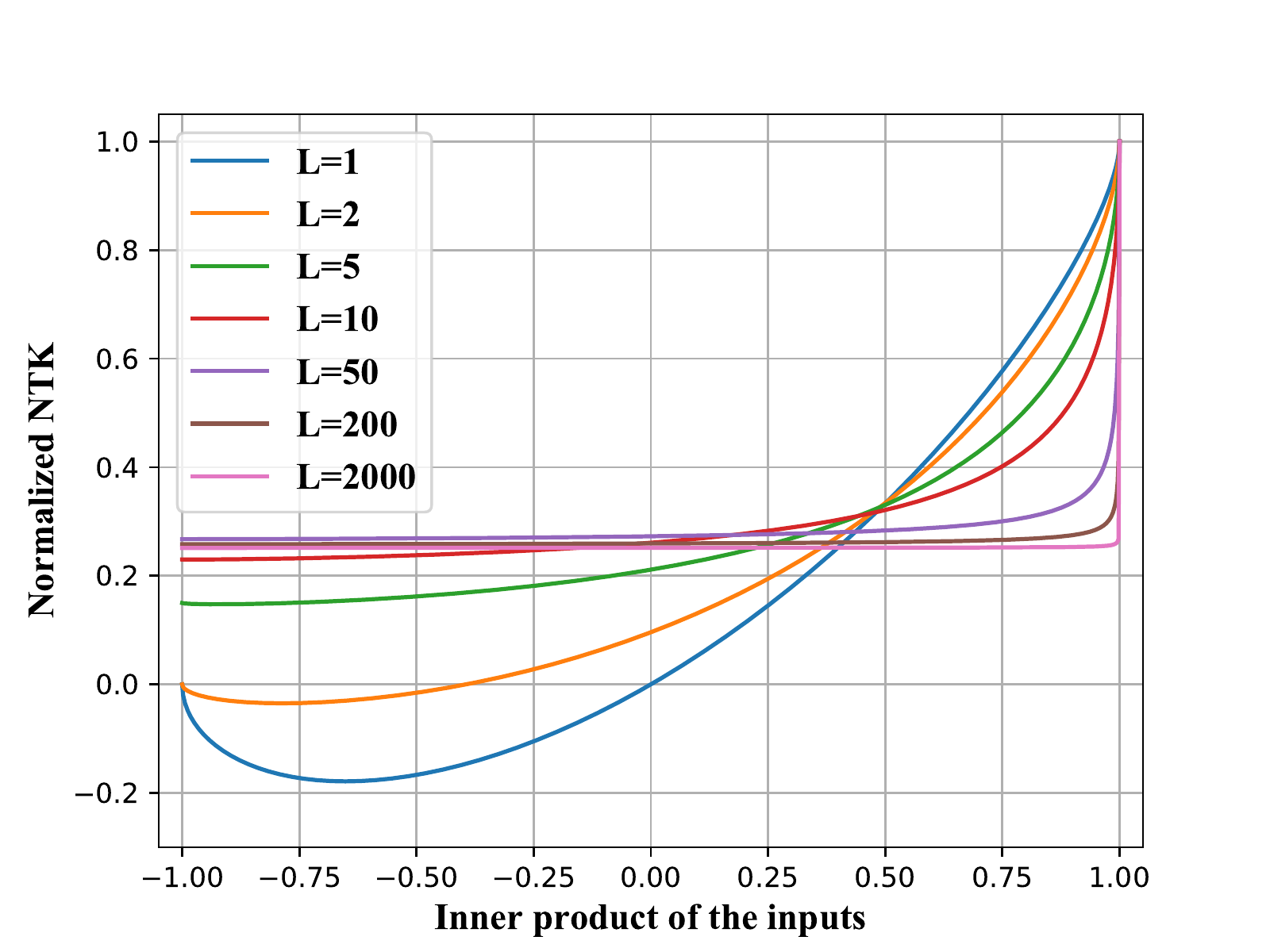}\label{ffntk}
        \vspace{-0.125in}
    }
    \subfigure[\it ResNets with $\gamma=1$]{
        \includegraphics[width=0.27\textwidth]{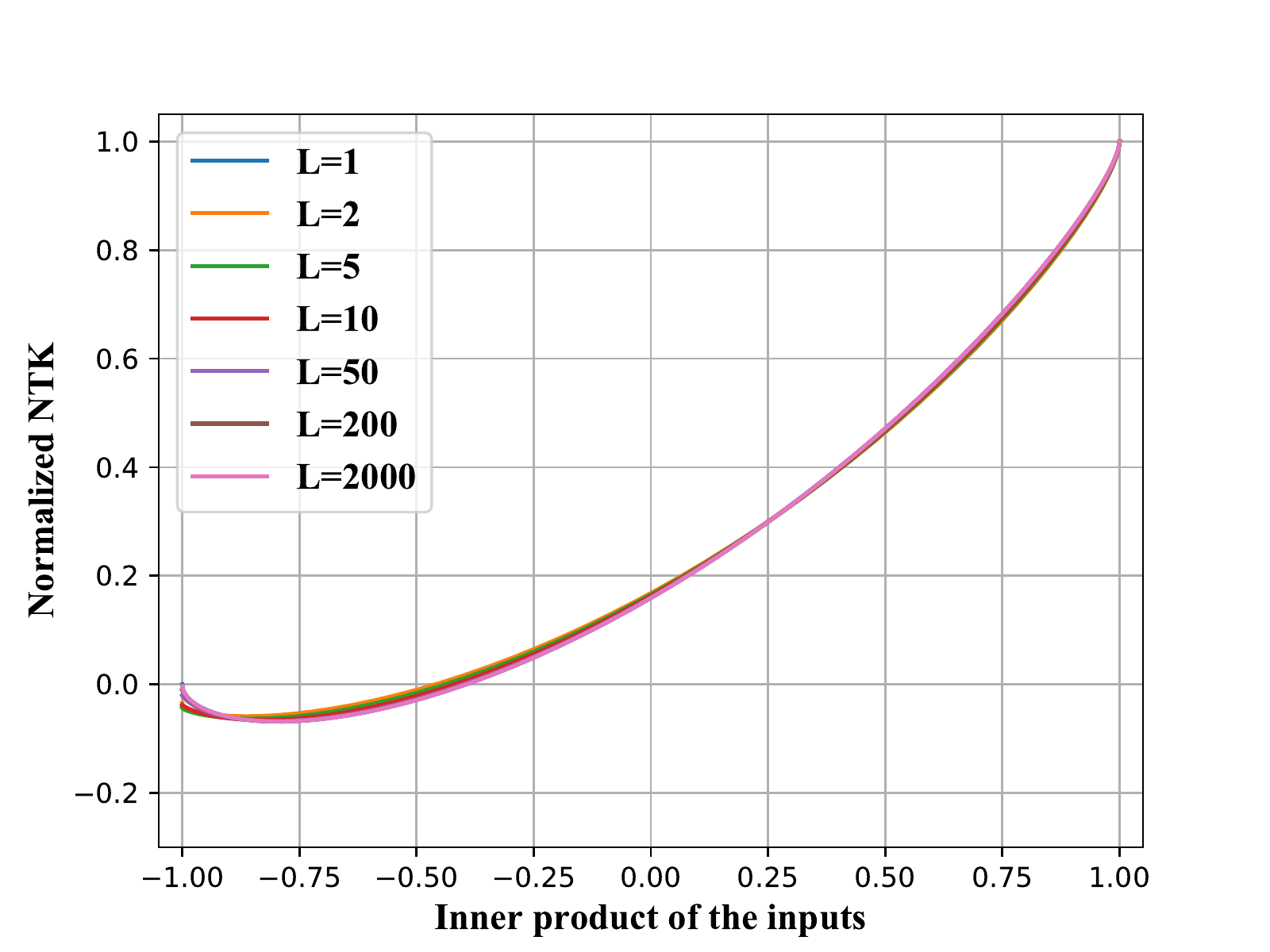}\label{resnetk}
        \vspace{-0.125in}
    }
    \subfigure[\it ResNets with $\gamma=0.5$]{
        \includegraphics[width=0.27\textwidth]{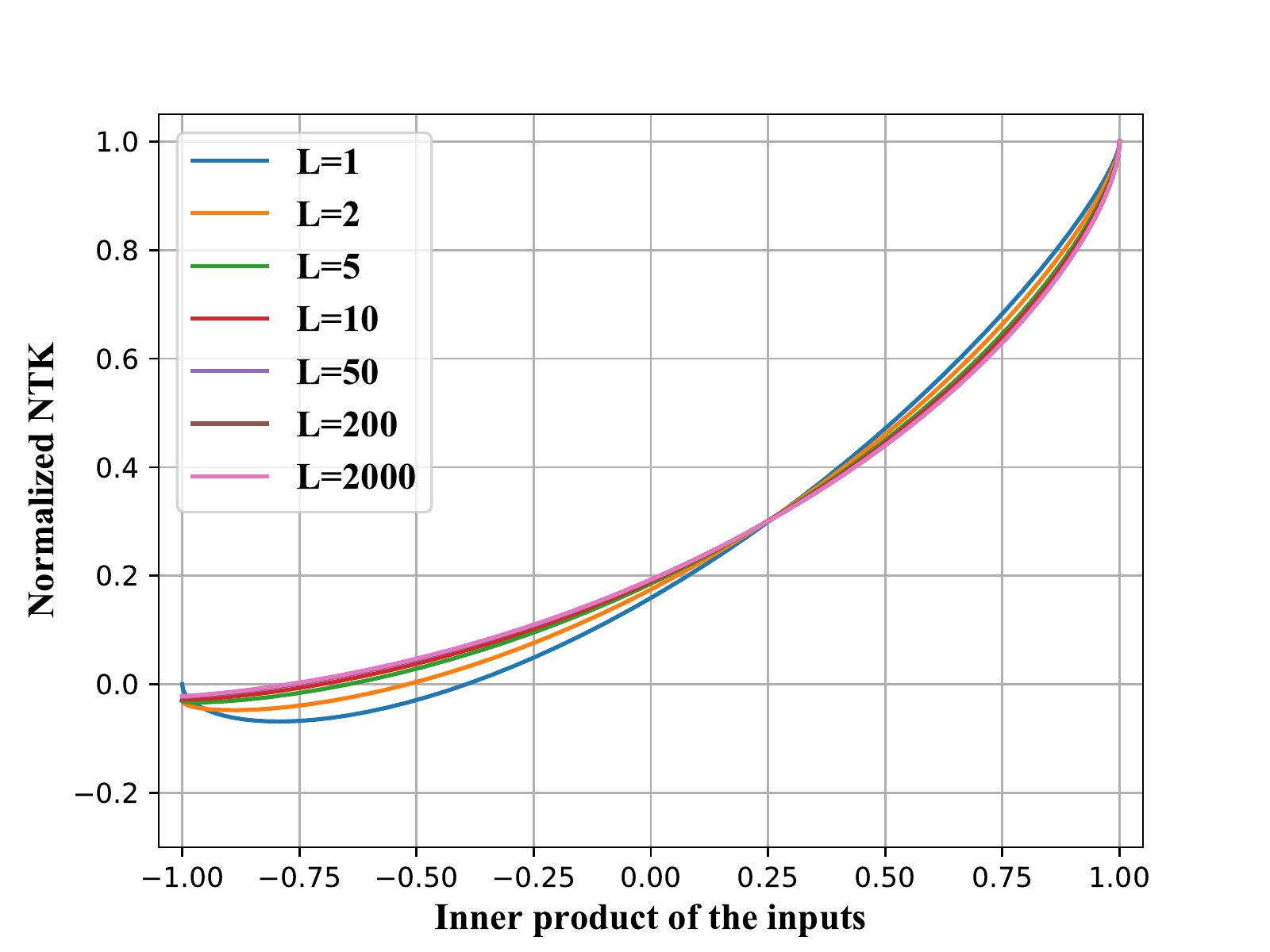}\label{resnetk5}
        \vspace{-0.125in}
    }
    \vspace{-0.125in}
    \caption{\it Normalized Neural Tangent Kernels Associated with Different Deep Networks.}
    \vspace{-0.15in}
\end{center}
\end{figure*}

Different from FFNets, the class of functions induced by the NTKs of the ResNets does not significantly change, as the depth $L$ increases. Surprisingly, we actually have $\overline{\Omega}_{\infty}=\overline{\Omega}_1$ for $\alpha=1/L$, i.e., infinitely deep and $1$-layer ResNets induce the same NTK. To further visualize such a difference, we plot the NTKs of the ResNets in Fig.~\ref{resnetk} and~\ref{resnetk5} for $\alpha=1/L$ and $\alpha=1/\sqrt{L}$, respectively. As can be seen, the increase of the depth yields very small changes to the NTKs. This partially explains why increasing the depth of the ResNet does not significantly deteriorate the generalization.

Moreover, as long as $x\ne \widetilde{x}$, i.e., $\langle x, \widetilde{x}\rangle\neq 1$, the limiting NTK of the FFNets always yields $1/4$ regardless how different $x$ is from $\widetilde{x}$. In contrast, the residual networks do not suffer from this drawback. The limiting NTK of the ResNets can greatly distinguish the difference between $x$ and $\widetilde{x}$, e.g., $\langle x, \widetilde{x}\rangle=-0.5$, $0$, and $0.5$ yield different values. Therefore, for an unseen data point, the corresponding regression model does not always output $0$, which is in sharp contrast to that of the limiting NTK of the FFNets.


\vspace{-0.15in}
\section{Experiments}
\label{sec:exp}
\vspace{-0.1in}

We demonstrate the generalization properties of the kernel regression based on the NTKs of the FFNets and the ResNets with varying depths. Our experiments follow similar settings to \cite{arora2019exact,arora2019fine}. We adopt two widely used data sets -- MNIST \citep{lecun1998mnist} and CIFAR10 \citep{krizhevsky2009cifar}, which are popular in existing literature. Note that both MNIST and CIFAR10 contains 10 classes of images. For simplicity, we select 2 classes out of 10 (digits ``0'' and ``8'' for MNIST, categories ``airplane'' and ``ship'' for CIFAR10), respectively, which results in two binary classification problems, denoted by MNIST2 and CIFAR2. 

Similar to \cite{arora2019exact,arora2019fine}, we use the kernel regression model for classification. Specifically, given the training data $(x_1,y_1),\cdots,(x_n,y_n)$, where $x_i\in\RR^D$ and $y_i\in\{-1,+1\}$ for $i=1,...,n$, we compute the kernel matrix $\tilde{K} = [\tilde{K}_{ij}]_{i,j=1}^n$ using the NTKs associated with the FFNets and the ResNets, where $\tilde{K}_{ij} = \overline{\Omega}_L(x_i,x_j)$. Then we compute the kernel regression function 
$
f(x) = \sum_{i=1}^n\alpha_i\overline{\Omega}_L(x,x_i),
$
where $[\alpha_1,...,\alpha_n]^\top=(\tilde{K}+\lambda I)^{-1}y$, $y=[y_1,...,y_n]^\top$ and $\lambda=0.1/n$ is a very small constant. We predict the label of $x$ to be $\sign(f(x))$.

Our experiments adopt the NTKs associated with three network architectures: (1) FFNets, (2) ResNets ($\gamma=0.5$) and (3) ResNets ($\gamma=1$). We set $n=200$ and $n=2000$. For each data set, we randomly select $n$ training data points ($n/2$ for each class) and $2000$ testing data points ($1000$ for each class). When training the kernel regression models, we normalize all training data points to have zero mean and unit norm. We repeat the procedure for 20 simulations. We find that the training errors of all simulations ($L$ varies from $1$ to $2000$) are $0.0$, which means that all NTK-based models are sufficient to overfit the training data, regardless $n=200$ or $n=2000$. The test accuracies of the kernel regression models with different kernels and depths are shown in Figure~\ref{fig:exp}.

As can be seen, the test accuracies of the kernel regression models of ResNets (both $\gamma=0.5$ and $\gamma=1$) are not sensitive to the depth. In contrast, the test accuracies of the kernel regression models of the FFNets significantly decrease, as the depth $L$ increases. Especially when the sample size is small ($n=200$), the kernel regression models behave like random guess for both MNIST2 and CIFAR2 when $L\geq 1000$. This is consistent with our analysis. 

\begin{figure}[htb!]
\vspace{-0.15in}
\begin{center}
    \subfigure[\it MNIST2 ($n=200$)]{
        \includegraphics[width=0.230\textwidth]{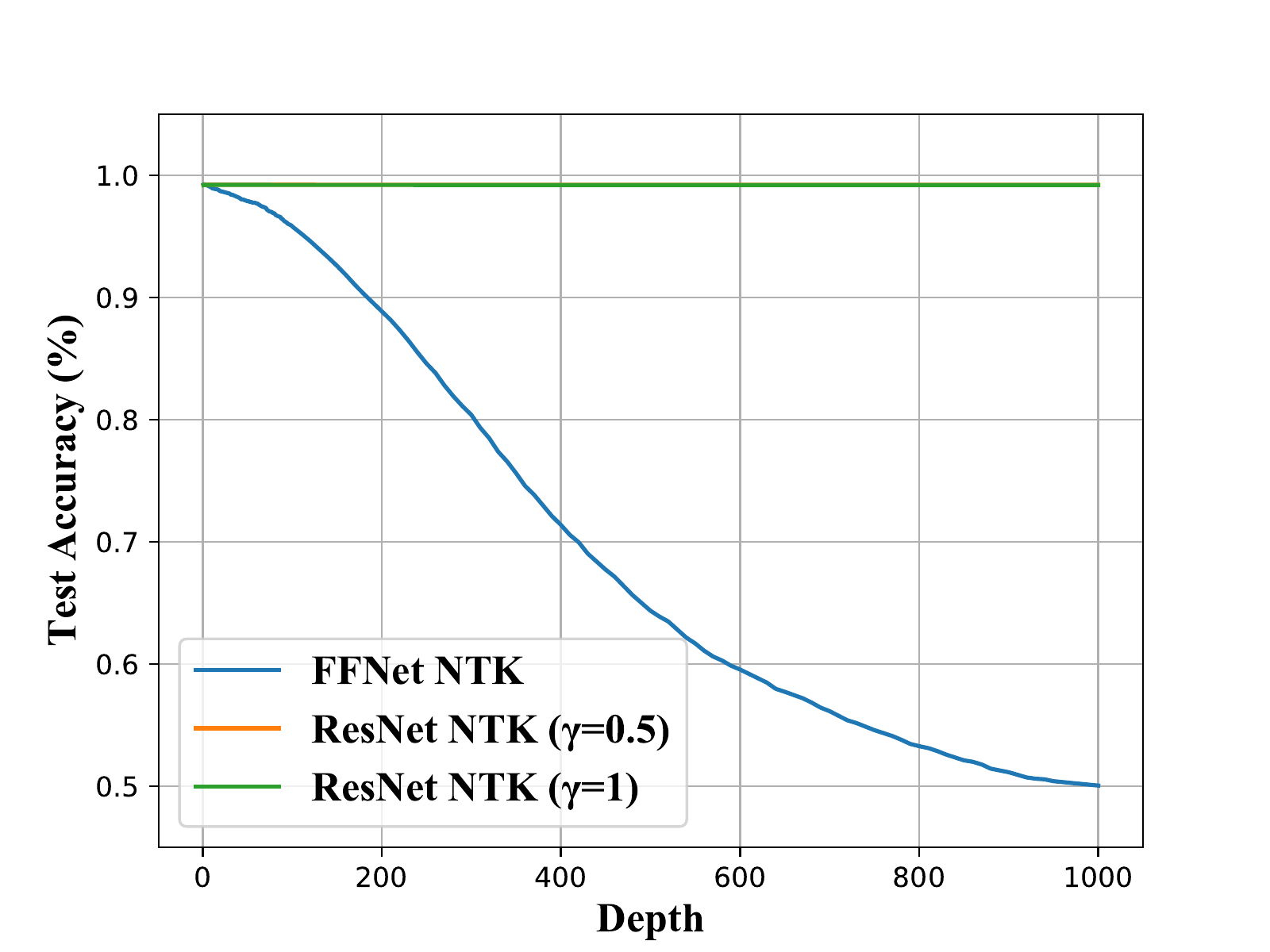}\label{fig:mnist_Ts_200}
        \vspace{-0.125in}
    }
    \subfigure[\it MNIST2 ($n=2000$)]{
        \includegraphics[width=0.230\textwidth]{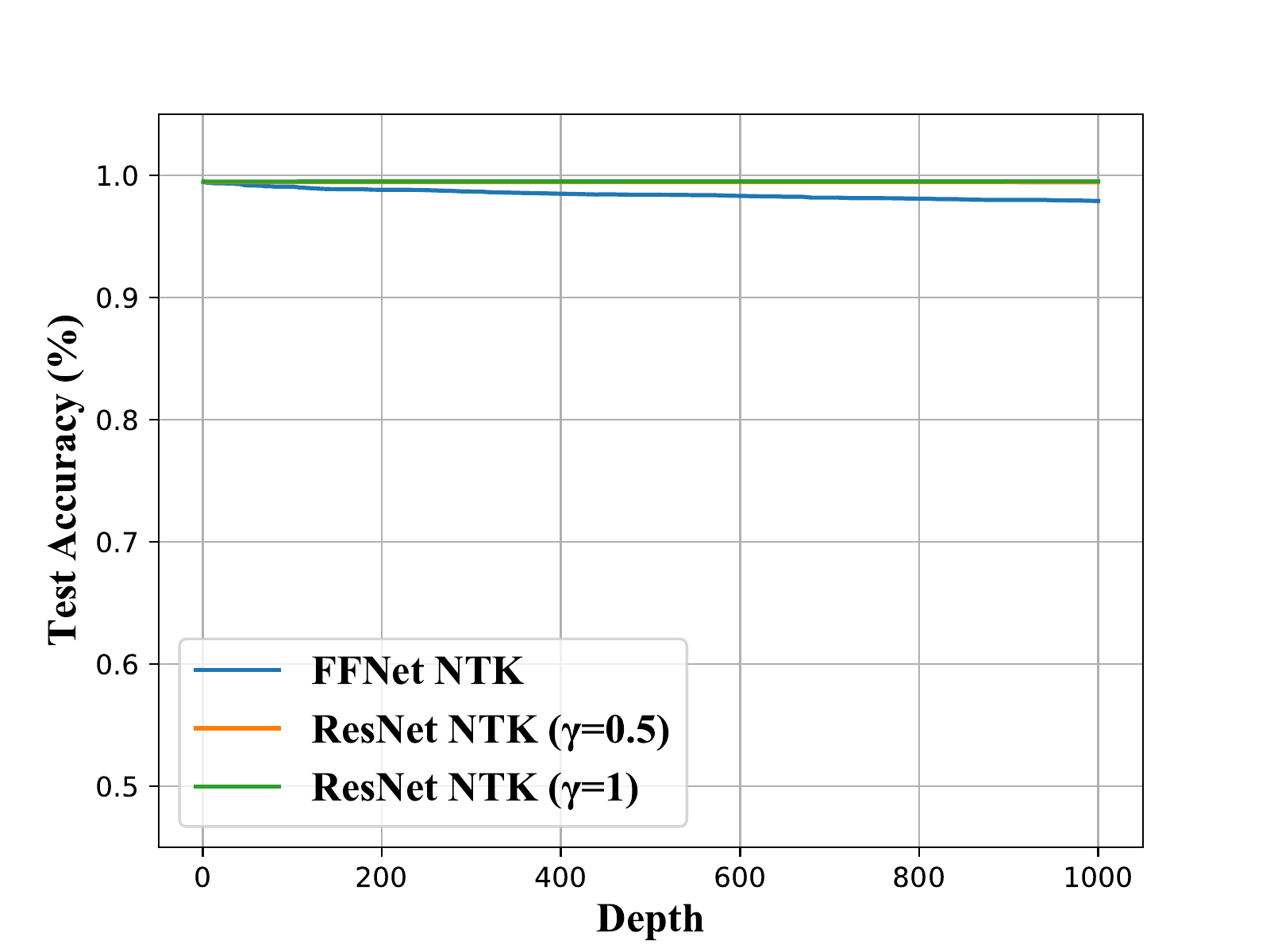}\label{fig:mnist_Ts_2000}
        \vspace{-0.125in}
    }
    \subfigure[\it CIFAR2 ($n=200$)]{
        \includegraphics[width=0.230\textwidth]{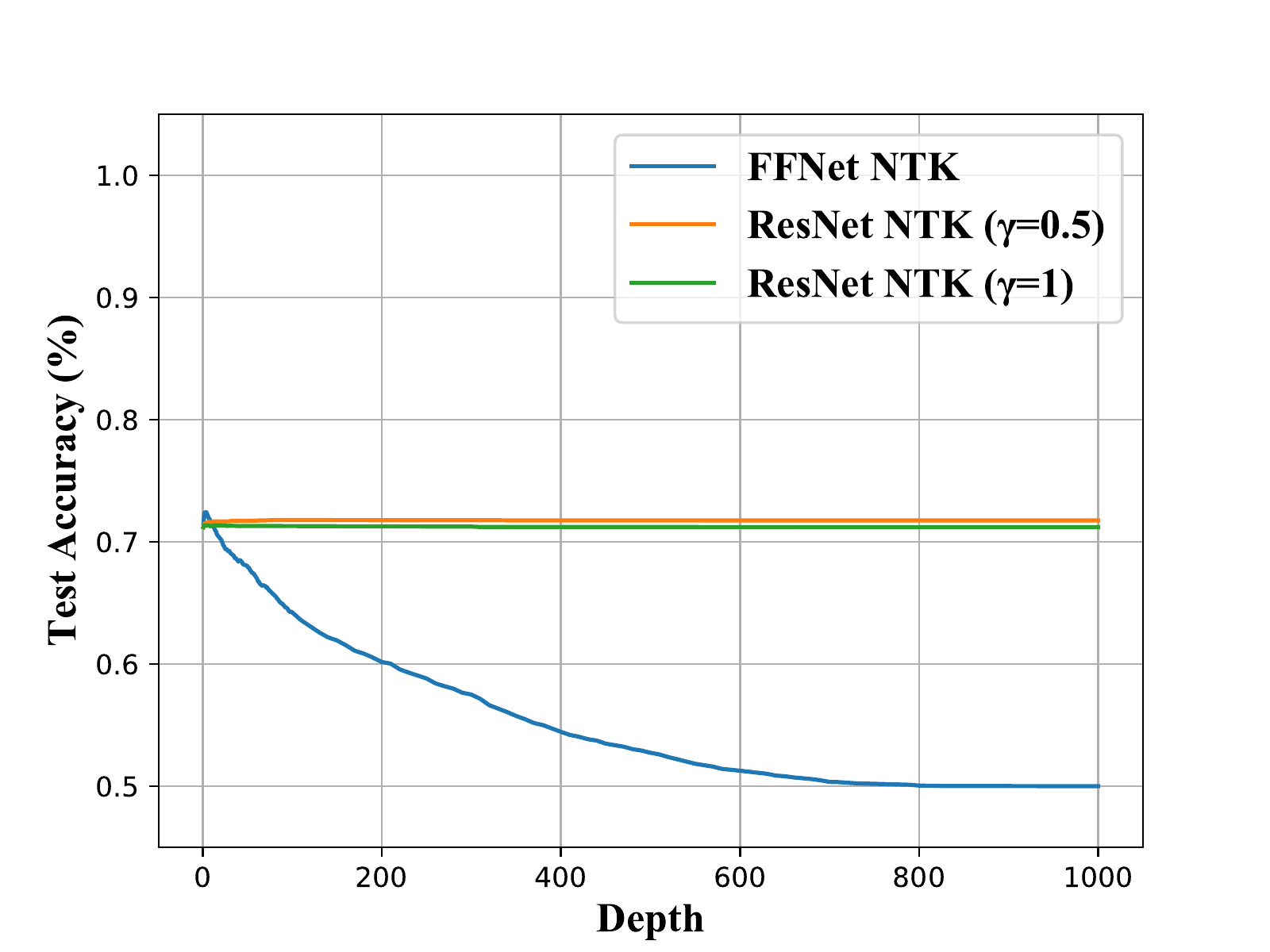}\label{fig:cifar10_Ts_200}
        \vspace{-0.125in}
    }
    \subfigure[\it CIFAR2 ($n=2000$)]{
        \includegraphics[width=0.230\textwidth]{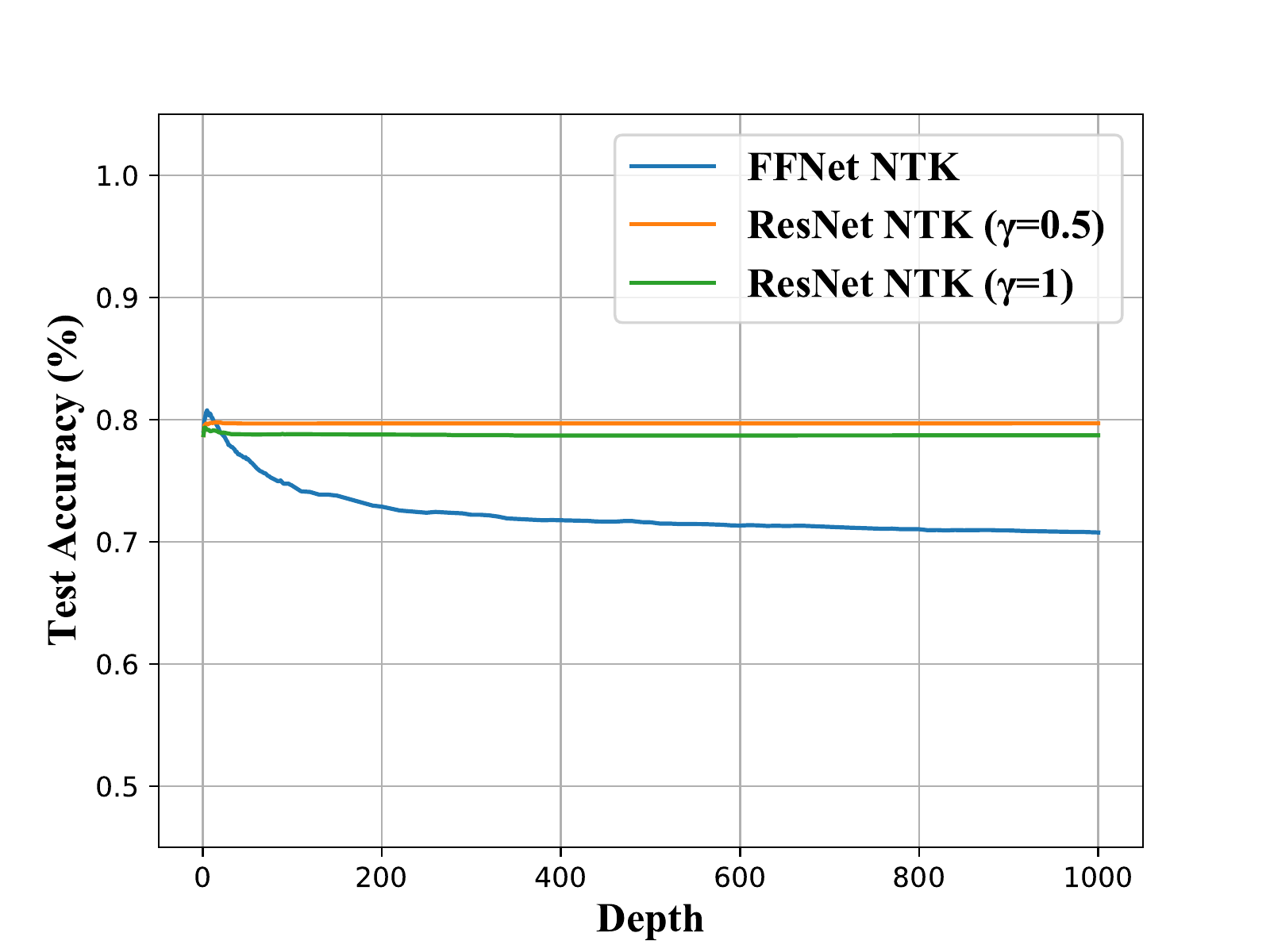}\label{fig:cifar10_Ts_2000}
        \vspace{-0.125in}
    }
    \vspace{-0.1in}
    \caption{\it Test accuracies of the kernel regression models evaluated on MNIST2 and CIFAR2.}
        \label{fig:exp}
    \vspace{-0.15in}
\end{center}
\end{figure}

Next we provide numerical verifications for our theorems. For Theorem~\ref{thm:ntk}, we randomly initialize the ResNet with width=500, scaling factor $\gamma=1$ and depth $L=5, 10, 100, 300$, and then calculate the inner product of the Jacobians of the ResNet for two different inputs as in the definition of NTK. We repeat the procedure for 500 times and plot the mean value (black cross) and the 1/4, 3/4 quantiles ("I"-shape line) of the sampled random NTKs and the theoretical NTK value in Fig.~\ref{fig:thm4}, which shows the two results match very well. For Theorem~\ref{thm:ffntkinf} and Theorem~\ref{thm:resntkinf}, Fig.~\ref{fig:thm5} and Fig.~\ref{fig:thm6} show that $\lim_{L\to\infty}|\overline{\Omega}_L(x,\tilde{x})-1/4|\cdot L/\rm{log}(L)\approx$ constant and $\lim_{L\to\infty}|\overline{\Omega}_L(x,\tilde{x})-\overline{\Omega}_1(x,\tilde{x})|\cdot L\approx$ constant with $x^\top\tilde{x}=K_0$ chosen at 9 points. 

\begin{figure}[htb!]
 \vspace{-0.2in}
\begin{center}
      \subfigure[\it Theorem~\ref{thm:ntk}]{\includegraphics[trim={0.2cm 0cm 0cm 1.2cm}, clip,width=0.28\linewidth]{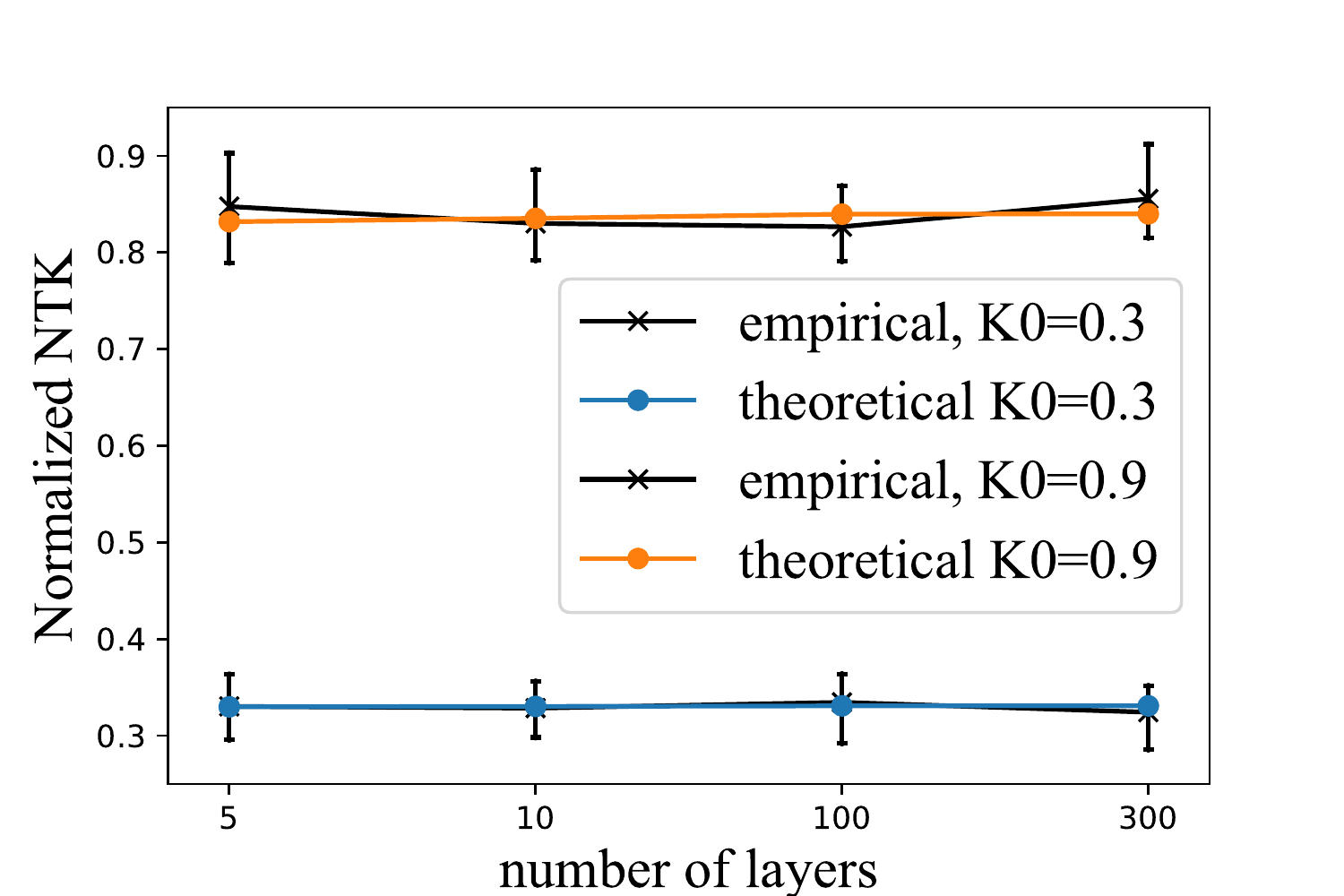}  
  \label{fig:thm4}

}
 \subfigure[\it Theorem~\ref{thm:ffntkinf}]{\includegraphics[trim={0.6cm 0cm 0.8cm 0cm}, clip,width=0.28\linewidth]{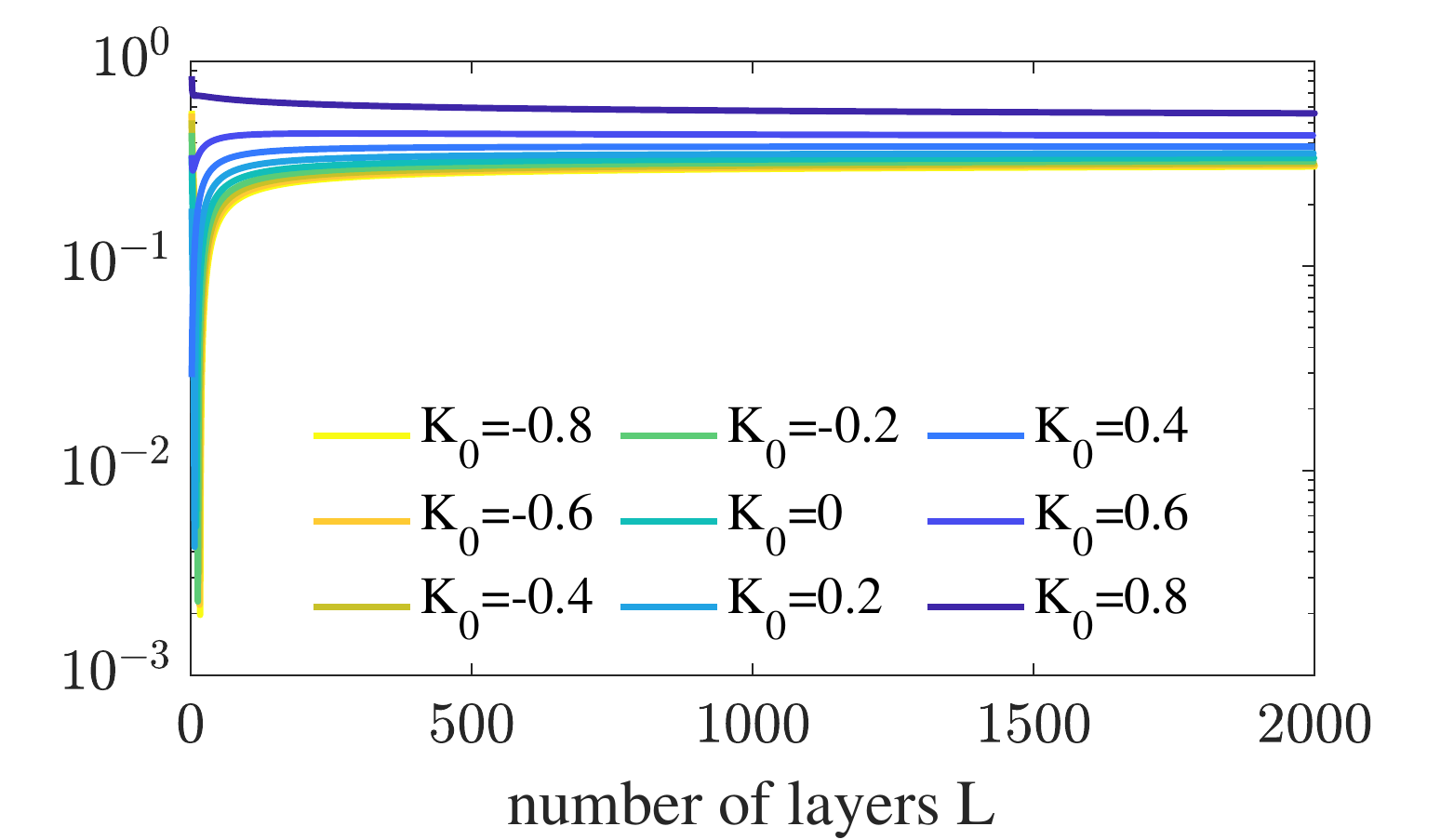}  
  \label{fig:thm5}

}
 \subfigure[\it Theorem~\ref{thm:resntkinf}]{\includegraphics[trim={0.6cm 0cm 0.8cm 0cm}, clip,width=0.28\linewidth]{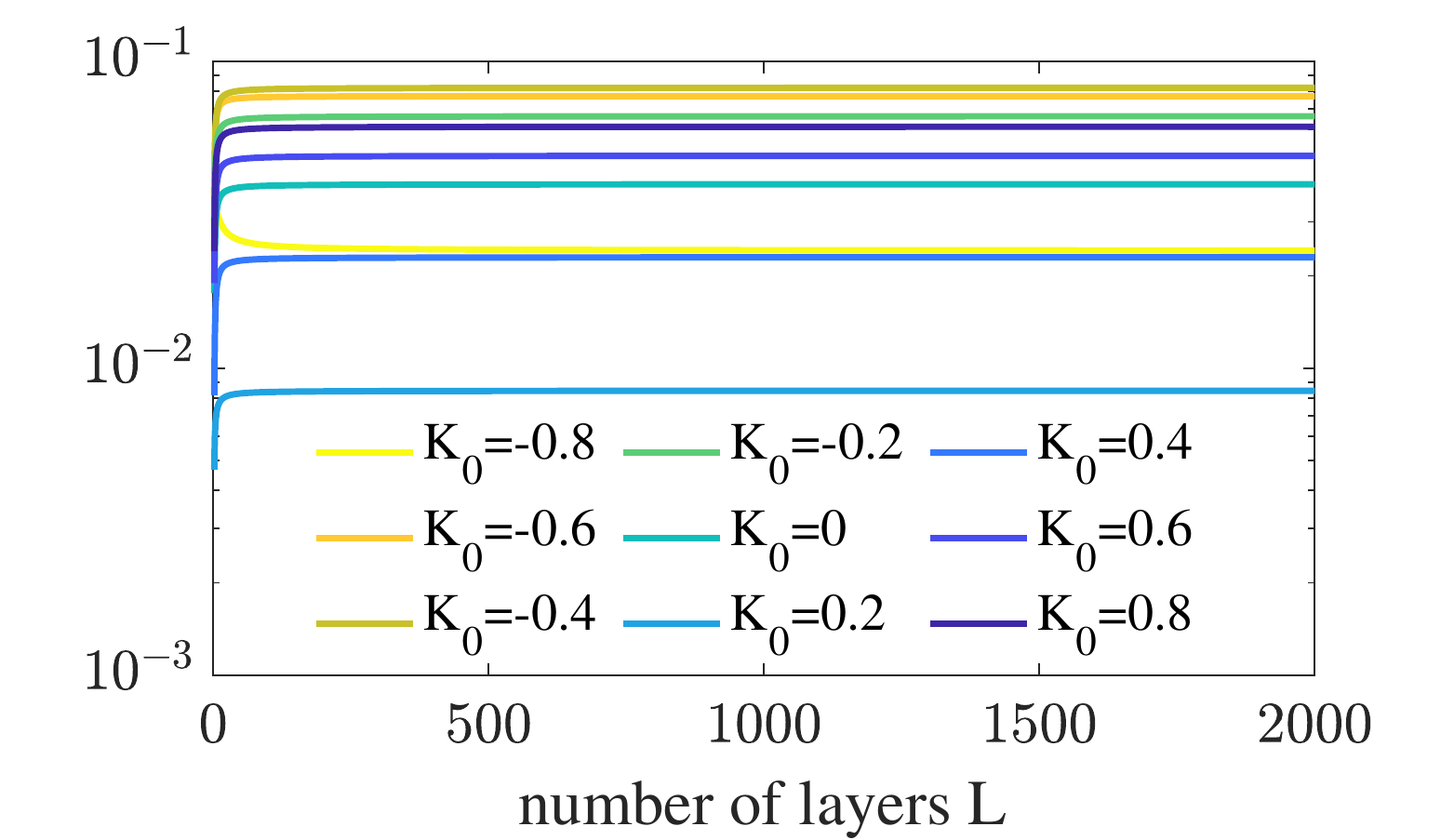}  
  \label{fig:thm6}
}
    \vspace{-0.1in}
\caption{\it Verification of main theorems. (a) Theorem~\ref{thm:ntk}, $m=500$ and scaling $\gamma=1$; (b) Theorem~\ref{thm:ffntkinf}, $y$-axis is $|\overline{\Omega}_L(x,\tilde{x})-1/4|\cdot L/\rm{log}(L)$; (c) Theorem~\ref{thm:resntkinf}, $y$-axis is $|\overline{\Omega}_L(x,\tilde{x})-\overline{\Omega}_1(x,\tilde{x})|\cdot L$}
\label{fig:fig}

\end{center}
\vspace{-0.25in}
\end{figure}


\vspace{-0.1in}
\section{Discussion}
\vspace{-0.15in}

\begin{wrapfigure}{r}{0.28\textwidth}
\begin{center}
\vspace{-0.45in}
 \includegraphics[width=0.28\textwidth]{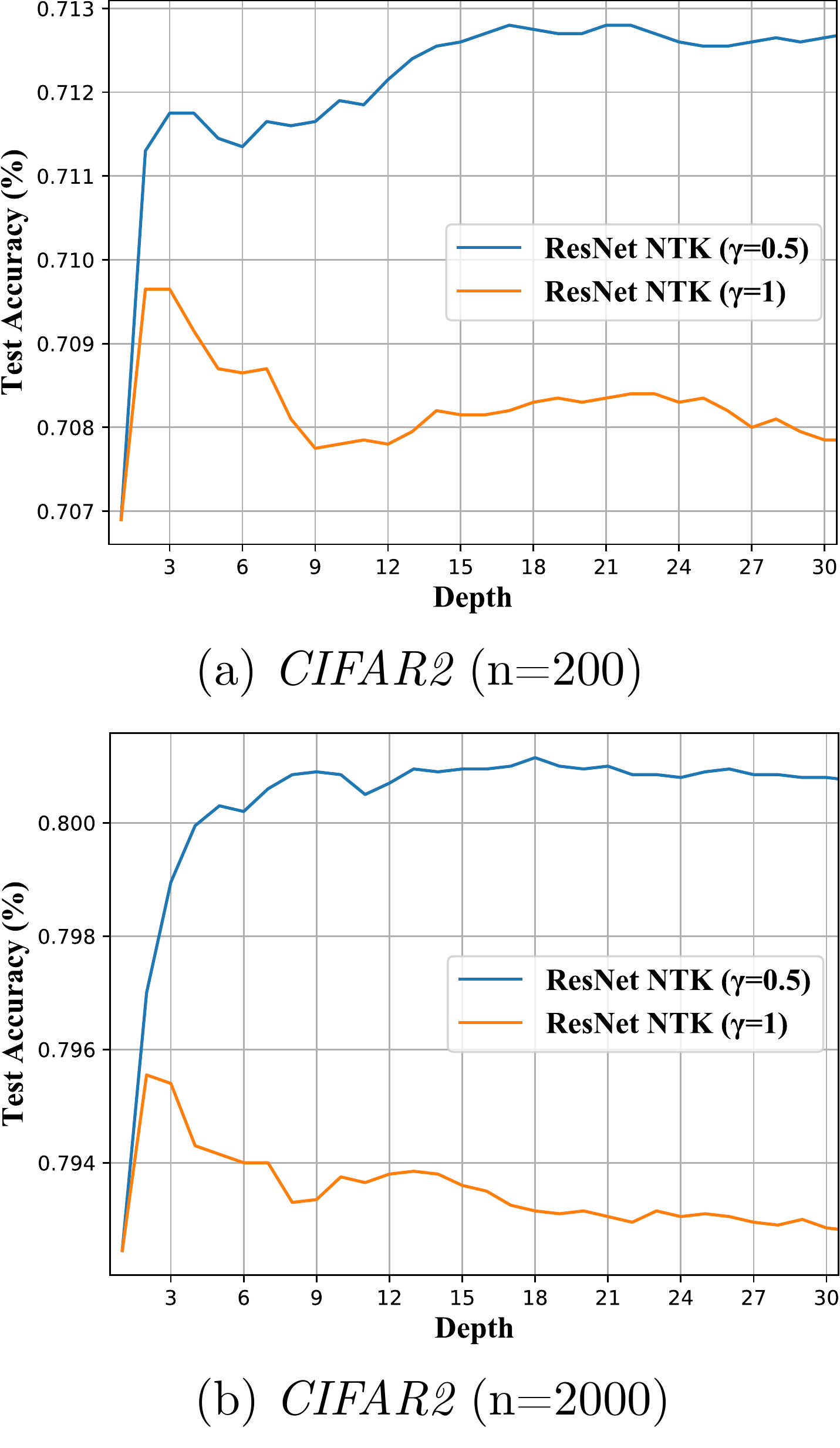}
\vspace{-0.3in}
\caption{\it Test accuracies of the kernel regression models evaluated on CIFAR2.}
\vspace{-0.4in}
\label{fig:cifar10_Ts_par}
\end{center}
\end{wrapfigure}

We discuss the NTK of the ResNet in more details. We remark unless specified, the NTK mentioned below indicates the normalized NTK.

Our theory shows the function class induced by the NTK of the deep ResNet asymptotically converges to that by the NTK of the 1-layer ResNet, as the depth increases. This indicates that the complexity of such a function class is not significantly different from that by the NTK of the 1-layer ResNet, for large enough $L$. Thus, the generalization gap does not significantly increase, as $L$ increases.

On the other hand, our experiments suggest that, as illustrated in Figure \ref{fig:cifar10_Ts_par}, the NTK of the ResNet with $\gamma=1$ actually achieves the best testing accuracy for CIFAR2 when $L=2$. The accuracy slightly decreases as $L$ increases, and becomes stable when $L\geq 9$. For the NTK of the ResNet with $\gamma=0.5$, the accuracy achieves the best when $L\approx 15$, and becomes stable for $L\geq 15$. Such evidence suggests that the function class induced by the NTKs of the ResNets with large $L$ and large $\gamma$ are possibly not as flexible as those by the NTKs of the deep ResNets with small $L$ and small $\gamma$.


Existing literature connects overparameterized neural networks to NTKs only under some very specific regime. Practical neural networks, however, are trained under more complicated regimes. Therefore, there still exists a significant theoretical gap between NTKs and practical neural networks. For example, Theorem~\ref{thm:resntkinf} shows that the NTK of the infinitely deep ResNet is identical to that of the 1-layer ResNet, while practical ResNets often show better generalization performance, as the depth increases. Also, we do not consider batch norm in our networks but refer to \citep{Jacot2019OrderAC} if necessary. We will leave these challenges for future investigation.

\newpage
\section*{Broader Impact}

This paper makes a significant contribution to extending the frontier of deep learning theory, and increases the intellectual rigor. To the best of our knowledge, our results are the first one for analyzing the effect of depth on the generalization of neural tangent kernels (NTKs). Moreover, our results are also the first one establishing the non-asymptotic bounds for NTKs of ResNets when all but the last layers are trained, which enables us to successfully analyze the generalization properties of ResNets through the perspective of NTK. This is in sharp contrast to the existing impractical theoretical results for NTKs of ResNets, which either only apply to an over-simplified structure of ResNets or only deal with the case when the last layer is trained.




\section*{Acknowledgement}

Molei Tao was partially supported by NSF DMS-1847802 and ECCS-1936776 and Yuqing Wang was partially supported by NSF DMS-1847802.
\bibliography{grant,icml2020}
\bibliographystyle{unsrt}

\appendix

\clearpage

\section{Proof of GP Kernels of ResNets}
\label{sec:proof_resgpk}
\subsection{Notation and Main Idea}
For a fixed pair of inputs $x$ and $\tx$, we introduce two matrices for each layer
\[
    \hat{\Sigma}_\ell (x,\tx) = \mattwo{\inner{x_{\ell }}{x_\ell }}{\inner{x_{\ell }}{\tx_\ell }}{\inner{\tx_{\ell }}{x_\ell }}{\inner{\tx_{\ell }}{\tx_\ell }},
\]
and
\[
    \Sigma_\ell (x,\tx) = \mattwo{K_\ell (x,x)}{K_\ell (x,\tx)}{K_\ell (\tx,x)}{K_\ell (\tx,\tx)}.
\]
$\hat{\Sigma}_\ell (x,\tx)$ is the empirical Gram matrix of the outputs of the $\ell$-th layer, while $\Sigma_\ell (x,\tx)$ is the infinite-width version. Theorem~\ref{thm:gp} says that with high probability, for each layer $\ell$, the difference of these two matrices measured by the entry-wise $L_\infty$ norm (denoted by $\|\cdot\|_{\max}$) is small.

The idea is to bound how much the $\ell$-th layer magnifies the input error to the output. Specifically, if the outputs of $(\ell-1)$-th layer satisfy
\[
 \left\|  \hat{\Sigma}_{\ell -1}(x,\tx) - \Sigma_{\ell -1}(x,\tx) \right\|_{\max} \leq \tau,
\]
we hope to prove that with high probability over the randomness of $W_{\ell }$ and $V_{\ell }$, we have
\[
 \left\|    \hat{\Sigma}_{\ell }(x,\tx) - \Sigma_{\ell }(x,\tx) \right\|_{\max} \leq \bigg(1+\cO\bigg(\frac{1}{L}\bigg) \bigg)\tau.
\]
Then the theorem is proved by first showing that w.h.p. $
 \left\|    \hat{\Sigma}_{0}(x,\tx) - \Sigma_{0}(x,\tx) \right\|_{\max} \leq (1+\cO(1/L))^{-L} \epsilon$ and then applying the result above for each layer.
\subsection{Lemmas}
We introduce the following lemmas. The first lemma shows the boundedness of $K_\ell (x,\tx)$.
\begin{lem}
For the ResNet defined in Eqn.~\eqref{eq:resnet_def}, $K_\ell (x,x) = (1+\alpha^2)^\ell$ for all $x\in \SSS^{D-1}$, $\ell=0,1,\cdots,L$. Also $K_\ell (x,x)$ is bounded uniformly when $0.5 \leq \gamma \leq 1$.
\end{lem}
Recall that $\phi_{W_\ell }(z) = \sqrt{\frac{2}{m}}\sigma_0(W_\ell  z)$. Since $W_\ell $ is Gaussian, we know that $\phi_{W_\ell }(x_{\ell -1})$ and $\phi_{W_\ell }(\tx_{\ell -1})$ are both sub-Gaussian random vectors over the randomness of $W_\ell $. Then their inner product enjoys sub-exponential property.
\begin{lem}[Sub-exponential concentration]
\label{lem:relu_concen}
With probability at least $1-\delta'$ over the randomness of $W_\ell  \sim \cN(0,I)$, when $m\geq c' \log(6/\delta') $, the following hold simultaneously 
\begin{align}
    &\Big| \langle \phi_{W_\ell } (x_{\ell -1}),  \phi_{W_\ell }(\tx_{\ell -1}) \rangle - \psi_\sigma(\hat{\Sigma}_{\ell -1} ( x , \tx ))\Big | \leq \sqrt{\frac{c'\log (6/\delta')}{m}} \|x_{\ell -1}\| \|\tx_{\ell -1}\|, \label{eq:con4}   \\ 
   & \Big| \| \phi_{W_\ell } (x_{\ell -1}) \|^2 - \|x_{\ell -1}\|^2 \Big| \leq \sqrt{\frac{c'\log (6/\delta')}{m}} \|x_{\ell -1}\|^2, \label{eq:con5} \\   
   & \Big| \| \phi_{W_\ell } (\tx_{\ell -1}) \|^2 - \|\tx_{\ell -1}\|^2 \Big| \leq \sqrt{\frac{c'\log (6/\delta')}{m}} \|\tx_{\ell -1}\|^2. \label{eq:con6}
\end{align}
\end{lem}
\begin{lem}[Locally Lipschitzness, based on \citep{daniely2016toward}]
\label{lem:relu_lip}
$\psi_\sigma$ is $(1+ \frac{1}{\pi}(\frac{r}{\mu})^2)$-Lipschitz w.r.t. ${\max}$ norm in $\mathcal{M}_{\mu,r} = \left\{ \mattwo{a}{b}{b}{c} |a,c \in [\mu-r,\mu+r] ; ac-b^2 >0 \right\}$ for all $\mu>0$, $0 < r \leq \mu/2$. That means, if (i). $\|\hat{\Sigma}_{\ell -1}(x,\tx) - \Sigma_{\ell -1}(x,\tx)\|_{\max} \leq \tau $ and (ii). $K_{\ell -1}(x,x) = K_{\ell -1}(\tx,\tx) = \mu$, for $\tau \leq \mu/2$, we have
\[
    \Big|\psi_\sigma(\hat{\Sigma}_{\ell -1}(x,\tx)) -  \psi_\sigma(\Sigma_{\ell -1}(x,\tx))\Big| \leq \Big (1 + \frac{1}{\pi}\Big(\frac{\tau}{\mu}\Big)^2 \Big) \tau.
\]
\end{lem}

\subsection{Proof of Theorem~\ref{thm:gp}}
\begin{proof}
In this proof, we also show the following hold with the same probability.
\begin{enumerate}
    \item For $\ell=0,1,\cdots,L$,  $\|x_\ell \|$ and $\|\tx_\ell \|$ are bounded by an absolute constant $C_1$ ($C_1 = 4$).
    \item For $\ell=1,\cdots,L$, $\|\phi_{W_\ell }(x_{\ell -1})\|$ and $\|\phi_{W_\ell }(\tx_{\ell -1})\|$ are bounded by an absolute constant $C_2$ ($C_2=8$).
    \item $\Big| \inner{\phi_{W_\ell }(x^{(1)}_{\ell -1})}{\phi_{W_\ell }(x^{(2)}_{\ell -1})} - \Gamma_\sigma(K_{\ell -1})(x^{(1)},x^{(2)}) \Big| \leq 2\epsilon$ for all $\ell=1,\cdots,L$ and $(x^{(1)},x^{(2)}) \in \{(x,x),(x,\tx),(\tx,\tx)\}$.
\end{enumerate}

We focus on the $\ell$-th layer. Let $\tau = \left\|  \hat{\Sigma}_{\ell -1}(x,\tx) - \Sigma_{\ell -1}(x,\tx) \right\|_{\max}$. Recall that $\Gamma_\sigma(K_{\ell -1})(x,\tx) = \psi_\sigma(\Sigma_{\ell -1}(x,\tx)) = \EE_{(X,\tX)\sim \cN(0,\Sigma_{\ell -1}(x,\tx))}\sigma(X)\sigma(\tX)$. Then 
\[K_\ell (x,\tx) = K_{\ell -1}(x,\tx) + \alpha^2 \psi_\sigma(\Sigma_{\ell -1}(x,\tx)).\]

Since $x_\ell  = x_{\ell -1} + \frac{\alpha}{\sqrt{m}}V_\ell  \phi_{W_\ell }(x_{\ell -1})$, we have
\begin{align*}
\langle x_\ell , \tx_\ell  \rangle &= \inner{x_{\ell -1}}{\tx_{\ell -1}} + \frac{\alpha^2}{m} \langle V_\ell \phi_{W_\ell }(x_{\ell -1}), V_\ell \phi_{W_\ell }(\tx_{\ell -1}) \rangle \\
& \qquad +\alpha \frac{1}{\sqrt{m}} \big( \langle V_\ell \phi_{W_\ell }(x_{\ell -1}), \tx_{\ell -1} \rangle + \langle V_\ell \phi_{W_\ell }(\tx_{\ell -1}), x_{\ell -1} \rangle   \big) \\
&=\inner{x_{\ell -1}}{\tx_{\ell -1}} + \alpha^2 P + \alpha(Q + R),
\end{align*}
where
\begin{align*}
P &\equiv\frac{1}{m} \langle V_\ell \phi_{W_\ell }(x_{\ell -1}), V_\ell \phi_{W_\ell }(\tx_{\ell -1}) \rangle, \\
Q &\equiv \frac{1}{\sqrt{m}} \big( \langle V_\ell \phi_{W_\ell }(x_{\ell -1}), \tx_{\ell -1} \rangle \big ), \\
R &\equiv \frac{1}{\sqrt{m}} \big( \langle V_\ell \phi_{W_\ell }(\tx_{\ell -1}), x_{\ell -1} \rangle \big ).    
\end{align*}
Under the randomness of $V_\ell $, $P$ is sub-exponential, and $Q$ and $R$ are Gaussian random variables. Therefore, for a given $\delta_0$, if $m \geq c_0\log(2/\delta_0)$, with probability at least $1-\delta_0$ over the randomness of $V_\ell $, we have 
\begin{equation}
\begin{aligned}
&\Big| P - \langle \phi_{W_\ell }(x_{\ell -1}) ,\phi_{W_\ell }(\tx_{\ell -1})  \rangle \Big| \leq \|\phi_{W_\ell }(x_{\ell -1}) \|  \|\phi_{W_\ell }(\tx_{\ell -1}) \| \sqrt{\frac{c_0 \log(2/\delta_0 )}{m}}; \label{eq:con1}    
\end{aligned}
\end{equation}
for a given $\tdelta$, with probability at least $1-2\tdelta$ over the randomness of $V_\ell $, we have 
\begin{equation}
| Q | \leq \|\phi_{W_\ell }(x_{\ell -1}) \|  \|\tx_{\ell -1} \| \sqrt{\frac{\tc \log(2/\tdelta )}{m}}, \label{eq:con2}    
\end{equation}
and 
\begin{equation}
| R | \leq \|\phi_{W_\ell }(\tx_{\ell -1}) \|  \|x_{\ell -1} \| \sqrt{\frac{\tc \log(2/\tdelta )}{m}}, \label{eq:con3}    
\end{equation}
where $c_0,\tc>0$ are absolute constants.

Using the above result and Lemma~\ref{lem:relu_concen} and setting $\delta_0 = \tdelta = \frac{\delta}{18(L+1)}$, $\delta' =\frac{\delta}{6(L+1)}$, when $m\geq C \log(36(L+1)/\delta)$, we have (\ref{eq:con1}), (\ref{eq:con2}), (\ref{eq:con3}), (\ref{eq:con4}), (\ref{eq:con5}), and (\ref{eq:con6}) hold with probability at least $1-\frac{\delta}{3(L+1)}$.

Recall that $\tau = \left\|  \hat{\Sigma}_{\ell -1}(x,\tx) - \Sigma_{\ell -1}(x,\tx) \right\|_{\max}$. Conditioned on $\tau < 0.5$, we have 
\[
    \|x_{\ell -1} \|^2 \leq K_{\ell -1}(x,x) + \tau \leq (1+\alpha^2)^L + \tau \leq e + \tau. 
\]
Similarly we can show $\|\tx_{\ell -1} \|^2$ is bounded by $e +\tau$. By \eqref{eq:con5} and \eqref{eq:con6} we have $\|\phi_{W_\ell } (x_{\ell -1}) \|^2 \leq 2 \|x_{\ell -1} \|^2$ and $\| \phi_{W_\ell } (\tx_{\ell -1}) \|^2 \leq 2 \|\tx_{\ell -1} \|^2$, which are both bounded.

Then
\begin{align*}
  \Big| \langle x_\ell , \tx_\ell  \rangle &- \left(\alpha^2 \psi_\sigma(\Sigma_{\ell -1}(x,\tx)) + K_{\ell -1}(x,\tx) \right) \Big|  \\
    & \leq \tau + \alpha^2 \big( P - \psi_\sigma(\Sigma_{\ell -1}(x,\tx))\big) + \alpha (|Q| +|R|)   \\
    & \leq \tau + \alpha^2 \Big| P -\langle \phi_{W_\ell }(x_{\ell -1}) ,\phi_{W_\ell }(\tx_{\ell -1})  \rangle \Big|    + \alpha \sqrt{\frac{\tc \log(2/\tdelta )}{m}}  \big( \|\phi_{W_\ell }(\tx_{\ell -1}) \|  \|x_{\ell -1} \| + \|\phi_{W_\ell }(x_{\ell -1}) \|  \|\tx_{\ell -1} \| \big) \\
    & \qquad+  \alpha^2 \Big|\psi_\sigma(\hat{\Sigma}_{\ell -1}(x,\tx)) -  \psi_\sigma(\Sigma_{\ell -1}(x,\tx))\Big|+ \alpha^2 \Big| \langle \phi_{W_\ell }(x_{\ell -1}) ,\phi_{W_\ell }(\tx_{\ell -1})  \rangle -  \psi_\sigma(\hat{\Sigma}_{\ell -1}(x,\tx))\Big| \\
    & \leq \tau + (\alpha^2 + \alpha) \sqrt{\frac{C_3 \log (36(L+1)/\delta)}{m}} + \alpha^2 \tau \bigg(1 + \frac{1}{\pi}\bigg(\frac{\tau}{K_{\ell -1}(x,x)}\bigg)^2\bigg)\\
    &\leq \tau + (\alpha^2 + \alpha) \sqrt{\frac{C_3 \log (36(L+1)/\delta)}{m}} + \alpha^2 \tau \bigg(1 + \frac{1}{4\pi}\bigg).
\end{align*}
When $\alpha = \frac{1}{L^\gamma}$, $\gamma \in [0.5,1]$, we have $\alpha^2 \leq 1/L$. Then when
\[
    m \geq \frac{C_3L^{2(1-\gamma)}\log(36(L+1)/\delta)}{\tau^2},
\]
we have
\[
    \Big| \langle x_\ell , \tx_\ell  \rangle - K_\ell (x,\tx) \Big| \leq \tau + \frac{4}{L}\tau.
\]
As a byproduct, we have
\begin{align*}
    &\Big|\inner{ \phi_{W_\ell }(x_{\ell -1}) }{\phi_{W_\ell }(\tx_{\ell -1})} - \psi_\sigma(\Sigma_{\ell -1}(x,\tx)) \Big| \\
    & \leq \sqrt{\frac{C_4 \log(36(L+1)/\delta)}{m}} + \Big (1+\frac{1}{\pi} \Big(\frac{\tau}{\mu}\Big)^2 \Big) \tau \leq 2\tau.
\end{align*}
Repeat the above for $(x_{\ell -1}, x_{\ell -1})$ and $(\tx_{\ell -1}, \tx_{\ell -1})$, we have with probability at least $1-\delta/(L+1)$ over the randomness of $V_\ell $ and $W_\ell $, 
\begin{equation}
\begin{aligned}
& \left\|  \hat{\Sigma}_{\ell -1}(x,\tx) - \Sigma_{\ell -1}(x,\tx) \right\|_{\max} \leq \tau \Rightarrow\\
& \quad \left\|    \hat{\Sigma}_{\ell }(x,\tx) - \Sigma_{\ell }(x,\tx) \right\|_{\max} \leq (1+4/L)\tau. \label{eq:induction}
 \end{aligned}
\end{equation}
Finally, when $m \geq \frac{C_5\log(6(L+1)/\delta)}{(\epsilon/e^4)^2}$, with probability at least $1-\delta/(L+1)$ over the randomness of $A$, we have 
\[
 \left\|  \hat{\Sigma}_{0}(x,\tx) - \Sigma_{0}(x,\tx) \right\|_{\max} \leq \epsilon/e^4.
\]
Then the result follows by successively using (\ref{eq:induction}).
\end{proof}

\subsection{proof of lemma~\ref{lem:relu_lip}}
\begin{proof}
\cite{daniely2016toward} showed that
\[
\left\|\nabla \psi_\sigma\mattwo{a}{b}{b}{c}\right\|_{1}=\frac{1}{2} \frac{a+c}{\sqrt{a c}}\left|\hat{\sigma}\left(\frac{b}{\sqrt{a c}}\right)-\frac{b}{\sqrt{a c}} \hat{\sigma}^{\prime}\left(\frac{b}{\sqrt{a c}}\right)\right|+\hat{\sigma}^{\prime}\left(\frac{b}{\sqrt{a c}}\right).
\]
When $a,c \in [\mu-r,\mu+r]$, we have
\begin{align*}
\frac{1}{2} \frac{a+c}{\sqrt{a c}} &=  \frac{1}{2}\bigg(
\sqrt{\frac{a}{c}} + \sqrt{\frac{c}{a}}
\bigg)  \leq \frac{1}{2} \bigg(
\sqrt{\frac{\mu+r}{\mu-r}} + \sqrt{\frac{\mu-r}{\mu+r}}
\bigg) = \bigg( 1- \bigg(\frac{r}{\mu}\bigg)^2 \bigg)^{-1/2} \leq 1 + \bigg(\frac{r}{\mu}\bigg)^2.
\end{align*}
The last inequality holds when $r< \frac{\mu}{2}$.

Define $\rho = \frac{b}{\sqrt{ac}}$, we have $\rho \in [-1,1]$. Then
\begin{align*}
\|\nabla \phi_\sigma\|_{1}&\leq  \bigg(1 + \bigg(\frac{r}{\mu}\bigg)^2\bigg)\Big|\hat{\sigma}\left(\rho\right)-\rho \hat{\sigma}^{\prime}\left(\rho\right)\Big|+\hat{\sigma}^{\prime}\left(\rho\right) \\
& = \bigg(1 + \bigg(\frac{r}{\mu}\bigg)^2\bigg)\left| \frac{\sqrt{1-\rho^2}}{\pi}\right|+ 1 - \frac{\cos^{-1}\rho}{\pi}\\
& \leq  \frac{\sqrt{1-\rho^2}}{\pi}+ 1 - \frac{\cos^{-1}\rho}{\pi} + \frac{1}{\pi} \bigg(\frac{r}{\mu}\bigg)^2  \\
& \leq 1 + \frac{1}{\pi} \bigg(\frac{r}{\mu}\bigg)^2 .
\end{align*} 

\end{proof}


\section{Proof of Theorem~\ref{thm:ntk}}
\label{sec:proof_resntk}
\subsection{Notation and Main Idea}
We already know that when the network width $m$ is large enough, $\inner{x_{\ell-1}}{\tx_{\ell-1}} \approx K_{\ell-1}(x,\tx)$, and $\inner{\phi_{W_\ell}( x_{\ell-1} )}{\phi_{W_\ell}( \tx_{\ell-1} )} \approx \Gamma_{\sigma}(K_{\ell-1})(x,\tx)$.

Next we need to show the concentration of the inner product of $\frac{b_\ell}{\sqrt{m}}$ and $\frac{\tb_\ell}{\sqrt{m}}$. We define two matrices for each layer
\[
    \hat{\Theta}_\ell(x,\tx) =\frac{1}{m} \mattwo{\inner{b_{\ell}}{b_\ell}}{\inner{b_{\ell}}{\tb_\ell}}{\inner{\tb_{\ell}}{b_\ell}}{\inner{\tb_{\ell}}{\tb_\ell}},
\]
and
\[
    \Theta_\ell(x,\tx) = \mattwo{B_\ell(x,x)}{B_\ell(x,\tx)}{B_\ell(\tx,x)}{B_\ell(\tx,\tx)}.
\]
Recall that 
\[
    b_{\ell} = \alpha \sqrt{\frac{1}{m}}\sqrt{\frac{2}{m}} W_{\ell}^\top D_{\ell}V_{\ell}^\top   b_{\ell+1} + b_{\ell+1}.
\]
We aim to show that when $\|\hat{\Theta}_{\ell+1}(x,\tx) - \Theta_{\ell+1}(x,\tx)\|_{\max} \leq \tau$, with high probability over the randomness of $W_\ell$ and $V_\ell$, we have $\|\hat{\Theta}_{\ell}(x,\tx) - \Theta_{\ell}(x,\tx)\|_{\max} \leq (1+\cO(1/L))\tau$. Notice that $b_{\ell+1}$ and $\tb_{\ell+1}$ contain the information of $W_\ell$ and $V_\ell$; they are not independent. Nevertheless we can decompose the randomness of $W_\ell$ and $V_\ell$ to show the concentration. This technique is also used in \cite{arora2019exact}.

\subsection{Lemmas}
In this part we introduce some useful lemmas. The first one shows the property of the step activation function.
\begin{lem}[Property of $\sigma'$]\citep{arora2019exact} 

\label{lem:prime_relu}
(1). Sub-Gaussian concentration. With probability at least $1-\delta$ over the randomness of $W_\ell$, we have
\[
\Big| \frac{2}{m}\Tr(D_\ell\tD_\ell) - \psi_{\sigma'}(\hat{\Sigma}_{\ell-1}(x,\tx)) \Big| \leq \sqrt{\frac{c\log(2/\delta)}{m}}.
\]
(2). Holder continuity. Fix $\mu>0,0<r\leq \mu$. For all $A,B \in \mathcal{M}_{\mu,r}= \bigg\{ \mattwo{a}{b}{b}{c} \bigg|a,c \in [\mu-r,\mu+r] ; ac-b^2 >0 \bigg\}$, if $\|A-B\|_{\max} \leq (\mu-r)\epsilon^2$, then \[
|\psi_{\sigma'}(A) - \psi_{\sigma'}(B)| \leq \epsilon.
\]
\end{lem}

The following lemma shows that regardless the fact that $b_{\ell+1}$ and $\tb_{\ell+1}$ depend on $V_\ell$, we can treat $V_\ell$ as a Gaussian matrix independent of $b_{\ell+1}$ and $\tb_{\ell+1}$ when the network width is large enough.
\begin{lem}\label{lem:key}
Assume the following inequality hold simultaneously for all $\ell=1,2,\cdots,L$
\[
\Big\|\frac{1}{\sqrt{m}}W_\ell \Big\| \leq C,\quad \Big\|\frac{1}{\sqrt{m}}V_\ell \Big\| \leq C.
\]
Fix an $\ell$. Further assume that 
\[
    \| \hat{\Theta}_{\ell+1}(x,\tx) - \Theta_{\ell+1}(x,\tx)\|_{\max} \leq 1.
\]
When $m\geq \max\{\frac{C}{\epsilon^2}(1+\log\frac{6}{\delta}),\frac{C}{\epsilon^2}\log\frac{8L}{\delta'}, cL^{2-2\gamma} \log \frac{8L}{\delta'}\}$, the following holds for all $(x^{(1)}, x^{(2)}) \in \{(x,x),(x,\tx),(\tx,\tx)\}$ with probability at least $1-\delta-\delta'$
\begin{align*}
    \bigg|\frac{2}{m} \frac{b_{\ell+1}^{(1)}}{\sqrt{m}}^\top V_\ell  D_\ell^{(1)} D_\ell^{(2)} V_\ell^\top \frac{b_{\ell+1}^{(2)}}{\sqrt{m}} - \inner{\frac{b_{\ell+1}^{(1)}}{\sqrt{m}}}{\frac{b_{\ell+1}^{(2)}}{\sqrt{m}}}\frac{2}{m}\Tr(D_\ell^{(1)}D_\ell^{(2)}) \bigg| \leq \epsilon.
\end{align*}

\end{lem}

The following lemma shows the same thing for $W_\ell$ as $V_\ell$ in Lemma~\ref{lem:key}.
\begin{lem}\label{lem:W} Assume the conditions and the results of Lemma~\ref{lem:key} hold. 

(1). When $m\geq \max\{\frac{C}{\epsilon^2}(1+\log\frac{6}{\delta}),\frac{C}{\epsilon^2}\log\frac{8L}{\delta'},  cL^{2-2\gamma} \log \frac{8L}{\delta'}\}$,the following holds for all $(x^{(1)}, x^{(2)}) \in \{(x,x),(x,\tx),(\tx,\tx)\}$ with probability at least $1-\delta-\delta'$
\[
\bigg|\frac{1}{m}\frac{2}{m} \inner{ W_{\ell}^\top D_{\ell}^{(1)}V_{\ell}^\top   \frac{b_{\ell+1}^{(1)}}{\sqrt{m}} }{ W_{\ell}^\top D_{\ell}^{(2)}V_{\ell}^\top   \frac{b_{\ell+1}^{(2)}}{\sqrt{m}} } - \frac{2}{m} \inner{ D_{\ell}^{(1)}V_{\ell}^\top   \frac{b_{\ell+1}^{(1)}}{\sqrt{m}} }{ D_{\ell}^{(2)}V_{\ell}^\top   \frac{b_{\ell+1}^{(2)}}{\sqrt{m}} } \bigg|\leq \epsilon.
\]

(2). When $m\geq \max \{ \frac{C}{\tepsilon^2}\log\frac{16L}{\tdelta}, cL^{2-2\gamma} \log \frac{16L}{\tdelta}\}$, for all $(x^{(1)},x^{(2)}) \in \{(x,x),(x,\tx),(\tx,x),(\tx,\tx)\}$, the following holds with probability at least $1-\tdelta$
\[
\bigg|\frac{1}{m}\sqrt{\frac{1}{m}}\sqrt{\frac{2}{m}} \inner{W_{\ell}^\top D_{\ell}^{(1)}V_{\ell}^\top   b_{\ell+1}^{(1)}}{b_{\ell+1}^{(2)}}\bigg| \leq \tepsilon.
\]

\end{lem}

\subsection{Proof of Theorem~\ref{thm:ntk}}
\begin{proof}
In this proof we are going to prove that when $m$ satisfies the assumption, with probability at least $1-\delta_0$, the following hold for $\ell = 1,\cdots,L$.
\begin{align*}
  &\bigg \vert \frac{1}{\alpha^2}\langle \dvl f, \dvl \tf \rangle -  B_{\ell +1}(x,\tx) \Gamma_{\sigma}(K_{\ell -1})(x,\tx)\bigg \vert \leq \epsilon_0,\\
    &\bigg \vert \frac{1}{\alpha^2}\langle \dwl f, \dwl \tf \rangle -  K_{\ell -1}(x,\tx)B_{\ell +1}(x,\tx)\Gamma_{\sigma'}(K_{\ell -1})(x,\tx)\bigg \vert \leq \epsilon_0.
\end{align*}

We break the proof into several steps. Each step is based on the result of the previous steps. Note that the absolute constants $c$ and $C$ may vary throughout the proof.

\textbf{Step 1. Norm Control of the Gaussian Matrices}

With probability at least $1-\delta_1$, when $m>c \log \frac{4L}{\delta_1}$, one can show that the following hold simultaneously for all $\ell=1,2,\cdots,L$ \citep{vershynin2010introduction}
\[
\bigg\|\frac{1}{\sqrt{m}}W_\ell\bigg \| \leq C,\quad \bigg\|\frac{1}{\sqrt{m}}V_\ell \bigg\| \leq C.
\]

\textbf{Step 2. Concentration of the GP kernels}

By Theorem~\ref{thm:gp}, with probability at least $1-\delta_2$, when
\[
    m \geq \frac{C}{\epsilon_2^4}L^{2-2\gamma} \log \frac{36(L+1)}{\delta_2},
\]
we have 
\begin{enumerate}
    \item For $\ell=0,\cdots,L$, $ \left\| \Sigma_{\ell}(x,\tx) - \hat{\Sigma}_\ell(x,\tx) \right\|_{\max} \leq c\epsilon_2^2$; 
    \item For $\ell=0,1,\cdots,L$,  $\|x_\ell\|$ and $\|\tx_\ell\|$ are bounded by an absolute constant $C_1$ ($C_1 = 4$);
    \item For $\ell=1,\cdots,L$, $\|\phi_{W_\ell}(x_{\ell-1})\|$ and $\|\phi_{W_\ell}(\tx_{\ell-1})\|$ are bounded by an absolute constant $C_2$ ($C_2=8$);
    \item $\Big| \inner{\phi_{W_\ell}(x^{(1)}_{\ell-1})}{\phi_{W_\ell}(x^{(2)}_{\ell-1})} - \Gamma_\sigma(K_{\ell-1})(x^{(1)},x^{(2)}) \Big| \leq 2c\epsilon_2^2$ for all $\ell=1,\cdots,L$ and $(x^{(1)},x^{(2)}) \in \{(x,x),(x,\tx),(\tx,\tx)\}$.
\end{enumerate}

\textbf{Step 3. Concentration of $\sigma'$}

By Lemma~\ref{lem:prime_relu}, when $m\geq \frac{C}{\epsilon_2^2}\log\frac{6L}{\delta_3}$, with probability at least $1-\delta_3$, for all $\ell=1,2,\cdots,L$ and $(x^{(1)}, x^{(2)}) \in \{(x,x),(x,\tx),(\tx,\tx)\}$, we have
\begin{align*}
    & \quad \Big|\frac{2}{m}\Tr(D_\ell^{(1)}D_\ell^{(2)}) - \Gamma_{\sigma'}(K_{\ell-1})(x^{(1)},x^{(2)}) \Big|\leq \sqrt{\frac{c\log(6L/\delta_3)}{m}} + \sqrt{2\left\|\hat{\Sigma}_{\ell-1}(x,\tx) - \Sigma_{\ell-1}(x,\tx)\right\|_{\max}} \leq \epsilon_2.
\end{align*}

\textbf{Step 4. Concentration of $B_\ell$}

Recall that 
\[
    b_{\ell+1}  = \left(v^\top \frac{\partial x_L}{\partial x_{L-1}} \frac{\partial x_{L-1}}{\partial x_{L-2}} \cdots \frac{\partial x_{\ell+1}}{\partial x_\ell}\right)^\top.
\]
We have
\[
 b_{L+1}= v,
\]
and for $\ell=1,2,\cdots,L-1$,
\begin{align*}
    b_{\ell+1} &= \frac{\partial x_{\ell+1}}{\partial x_\ell}^\top b_{\ell+2} = \alpha \sqrt{\frac{1}{m}}\sqrt{\frac{2}{m}} W_{\ell+1}^\top D_{\ell+1}V_{\ell+1}^\top   b_{\ell+2} + b_{\ell+2}.
\end{align*}
Following the same idea in Thm~\ref{thm:gp}, we prove by induction. First of all, for $b_{L+1}$, we have $\Theta_{L+1}(x,\tx) = \mattwo{1}{1}{1}{1}, \hat{\Theta}_{L+1}(x,\tx) = \frac{\|v\|^2}{m}\mattwo{1}{1}{1}{1}$. Then by Bernstein inequality \citep{mohri2018foundations}, with probability at least $1-\frac{\delta_4}{L}$, when $m\geq \frac{C}{\epsilon_4^2}\log\frac{2L}{\delta_4}$, we have
\[
\left\vert \frac{\|v\|^2}{m} - 1 \right\vert \leq \epsilon_4.
\]
Fix $\ell \in \{2,3,\cdots,L\}$. 
Assume that
\[
    \left\| \hat{\Theta}_{\ell+1}(x,\tx) - \Theta_{\ell+1}(x,\tx)\right\|_{\max} \leq \tau \leq 1,
\]
we hope to prove with high probability, 
\[
    \left\| \hat{\Theta}_{\ell}(x,\tx) - \Theta_{\ell}(x,\tx)\right\|_{\max} \leq (1+\cO(1/L))\tau.
\]
First write
\begin{align*}
    \frac{1}{m}\inner{b_{\ell}^{(1)}}{b_{\ell}^{(2)}} &= \frac{1}{m}\inner{b_{\ell+1}^{(1)}}{b_{\ell+1}^{(2)}} + \alpha^2 P + \alpha(Q+R),
\end{align*}
where
\begin{align*}
    P &= \frac{1}{m}\frac{2}{m} \inner{ W_{\ell}^\top D_{\ell}^{(1)}V_{\ell}^\top   \frac{b_{\ell+1}^{(1)}}{\sqrt{m}} }{ W_{\ell}^\top D_{\ell}^{(2)}V_{\ell}^\top   \frac{b_{\ell+1}^{(2)}}{\sqrt{m}} }, \\
    Q & = \frac{1}{m}\sqrt{\frac{1}{m}}\sqrt{\frac{2}{m}} \inner{W_{\ell}^\top D_{\ell}^{(1)}V_{\ell}^\top   b_{\ell+1}^{(1)}}{b_{\ell+1}^{(2)}},\\
    R & = \frac{1}{m}\sqrt{\frac{1}{m}}\sqrt{\frac{2}{m}} \inner{W_{\ell}^\top D_{\ell}^{(2)}V_{\ell}^\top   b_{\ell+1}^{(2)}}{b_{\ell+1}^{(1)}}.
\end{align*}
Then 
\begin{align*}
   &\quad \Big|\frac{1}{m}\inner{b_{\ell}^{(1)}}{b_{\ell}^{(2)}} - (B_{\ell+1}(x^{(1)},x^{(2)}) + \alpha^2 B_{\ell+1}(x^{(1)},x^{(2)})\Gamma_{\sigma'}(K_{\ell-1})(x^{(1)},x^{(2)})\Big|\\
    &\leq \Big|\frac{1}{m}\inner{b_{\ell+1}^{(1)}}{b_{\ell+1}^{(2)}} - B_{\ell+1}(x^{(1)},x^{(2)})\Big| + \alpha^2 \Big|P - B_{\ell+1}(x^{(1)},x^{(2)})\Gamma_{\sigma'}(K_{\ell-1})(x^{(1)},x^{(2)})\Big| + \alpha |Q|+\alpha|R| \\
    &\leq \tau + \alpha^2 \Big|P - \frac{2}{m} \inner{ D_{\ell}^{(1)}V_{\ell}^\top   \frac{b_{\ell+1}^{(1)}}{\sqrt{m}} }{ D_{\ell}^{(2)}V_{\ell}^\top   \frac{b_{\ell+1}^{(2)}}{\sqrt{m}} }\Big| \\
    & \qquad+ \alpha^2\Big|\frac{2}{m} \inner{ D_{\ell}^{(1)}V_{\ell}^\top   \frac{b_{\ell+1}^{(1)}}{\sqrt{m}} }{ D_{\ell}^{(2)}V_{\ell}^\top   \frac{b_{\ell+1}^{(2)}}{\sqrt{m}} } - \inner{\frac{b_{\ell+1}^{(1)}}{\sqrt{m}}}{\frac{b_{\ell+1}^{(2)}}{\sqrt{m}}}\frac{2}{m}\Tr(D_\ell^{(1)}D_\ell^{(2)})\Big| \\
    &\qquad+\alpha^2\Big|\inner{\frac{b_{\ell+1}^{(1)}}{\sqrt{m}}}{\frac{b_{\ell+1}^{(2)}}{\sqrt{m}}} - B_{\ell+1}(x^{(1)},x^{(2)}) \Big|\Big| \frac{2}{m}\Tr(D_\ell^{(1)}D_\ell^{(2)})\Big|\\
     &\qquad+ \alpha^2\Big|B_{\ell+1}(x^{(1)},x^{(2)})\Big|\Big|\frac{2}{m}\Tr(D_\ell^{(1)}D_\ell^{(2)}) - \Gamma_{\sigma'}(K_{\ell-1})(x^{(1)},x^{(2)})\Big|  \\
     &\qquad+ \alpha|Q| + \alpha|R|.
\end{align*}
In Lemma~\ref{lem:key} and Lemma~\ref{lem:W}, set $\tepsilon=cL^{\gamma-1}\tau $, $\epsilon = c\tau$, $\delta = \tdelta = \delta' = \delta_4/5L$. When $m\geq \max\{\frac{C}{\tau^2}(1+\log\frac{30L}{\delta_4}),\frac{C}{\tau^2}\log\frac{40L^2}{\delta_4}, \frac{C}{\tau^2}L^{2-2\gamma}\log\frac{80L^2}{\delta_4}, cL^{2-2\gamma}\log \frac{80L^2}{\delta_4}\}$, with probability at least $1-\frac{\delta_4}{L}$, the results of Lemma~\ref{lem:key} and Lemma~\ref{lem:W} hold.
Then for all $(x^{(1)},x^{(2)}) \in \{(x,x),(x,\tx),(\tx,\tx)\}$,
\begin{align*}
    \Big|\frac{1}{m}\inner{b_\ell^{(1)}}{b_\ell^{(2)}} - B_\ell(x^{(1)},x^{(2)})\Big| &\leq \tau + \alpha^2 c\tau + \alpha^2 c \tau + \alpha^2 2\tau + \alpha^2 e \epsilon_2 + 2 \alpha cL^{a-1}\tau \\
    &\leq \tau (1+\cO(1/L)). \quad  (\text{Set } \epsilon_2 \leq c \tau.)
\end{align*}

By taking union bound, with probability at least $1-\delta_4$, we have for all $\ell=1,2,\cdots, L$,
\[
\|\hat{\Theta}_{\ell+1}(x,\tx) - \Theta_{\ell+1}(x,\tx)\|_{\max} \leq (1+\cO(1/L))^L \epsilon_4 \leq C \epsilon_4.
\]
Meanwhile, we have for all $(x^{(1)},x^{(2)}) \in \{(x,x),(x,\tx),(\tx,\tx)\}$ and $\ell=1,\cdots,L$,
\[
\bigg| \frac{2}{m} \inner{ D_{\ell}^{(1)}V_{\ell}^\top   \frac{b_{\ell+1}^{(1)}}{\sqrt{m}} }{ D_{\ell}^{(2)}V_{\ell}^\top   \frac{b_{\ell+1}^{(2)}}{\sqrt{m}} } - B_{\ell+1}(x^{(1)},x^{(2)})\Gamma_{\sigma'}(K_{\ell-1})(x^{(1)},x^{(2)}) \bigg| \leq (2+c)\tau +e\epsilon_2 \leq C \epsilon_4.
\]

\textbf{Step 5. Summary}

Using previous results, for all $\ell$, we have
\begin{align*}
    &\quad \Big|\frac{1}{\alpha^2}\langle \dvl f, \dvl \tf \rangle -  B_{\ell+1} \Gamma_{\sigma}(K_{\ell-1})\Big| \\
    &\leq \Big|\frac{1}{m} \inner{b_{\ell+1}}{\tb_{\ell+1}} - B_{\ell+1} \Big|\cdot | \inner{\phi_{W_\ell}(x_{\ell-1})}{\phi_{W_\ell}(\tx_{\ell-1})}| + | B_{\ell+1} |\cdot|\inner{\phi_{W_\ell}(x_{\ell-1})}{\phi_{W_\ell}(\tx_{\ell-1})}- \Gamma_{\sigma}(K_{\ell-1})|\\
    & \leq C\epsilon_4 +C \epsilon_2^2,
\end{align*}
and
\begin{align*}
    &\quad \Big|\frac{1}{\alpha^2}\langle \dwl f, \dwl \tf \rangle -  K_{\ell-1}B_{\ell+1}\Gamma_{\sigma'}(K_{\ell-1})\Big|\\
    &\leq \Big|\frac{1}{m}  \inner{x_{\ell-1}}{\tx_{\ell-1}}- K_{\ell-1}\Big| \cdot \Big|\frac{2}{m} \tb_{\ell+1} ^\top V_\ell  \tD_\ell D_\ell V_\ell^\top b_{\ell+1}\Big| +|  K_{\ell-1}|\cdot\Big| \frac{2}{m} \tb_{\ell+1} ^\top V_\ell  \tD_\ell D_\ell V_\ell^\top b_{\ell+1} - B_{\ell+1}\Gamma_{\sigma'}(K_{\ell-1})\Big| \\
    & \leq C\epsilon_2^2 + C\epsilon_4.
\end{align*}
To sum up, by choosing $\epsilon_4 = c\epsilon_0$, $\epsilon_2 = c\epsilon_4$, and $\delta_1=\delta_2=\delta_3=\delta_4=\delta_0/4$, then with probability at least $1-\delta_0$, when
\begin{align*}
    m&\geq \frac{C}{\epsilon_0^4}L^{2-2\gamma}\bigg(\log\frac{320(L^2+1)}{\delta_0}+1\bigg)\\
    &\geq \max \bigg\{c\log\frac{16L}{\delta_0}, \frac{C}{\epsilon_0^4}L^{2-2\gamma}\log\frac{144(L+1)}{\delta_0}, \frac{C}{\epsilon_0^2}\log\frac{24L}{\delta_0},\\
    &\qquad \frac{C}{\epsilon_0^2}\log\frac{8L}{\delta_0}, \frac{C}{\epsilon_0^2}(1+\log\frac{120L}{\delta_0}),\frac{C}{\epsilon_0^2}\log\frac{160L^2}{\delta_0}, \frac{C}{\epsilon_0^2}L^{2-2\gamma}\log\frac{320L^2}{\delta_0},cL^{2-2\gamma}\log\frac{320L^2}{\delta_0^4} \bigg\},\\
\end{align*}
the desired results hold.

\end{proof}

\section{Proofs of the Lemmas}

\subsection{Supporting lemmas}
\begin{lem}\label{lem:projGperp}
Define $G = [\phi_{W_\ell}(x_{\ell-1}), \phi_{W_\ell}(\tx_{\ell-1})]$, and $\Pi_G^\perp$ as the orthogonal projection onto the orthogonal complement of the column space of $G$. when $m\geq 1+\log\frac{6}{\delta}$, the following holds with probability at least $1-\delta$ for all $(x^{(1)}, x^{(2)}) \in \{(x,x),(x,\tx),(\tx,\tx)\}$,
\begin{align*}
    &\bigg|\frac{2}{m} \frac{b_{\ell+1}^{(1)}}{\sqrt{m}}^\top V_\ell \Pi_G^\perp D_\ell^{(1)} D_\ell^{(2)} \Pi_G^\perp V_\ell^\top \frac{b_{\ell+1}^{(2)}}{\sqrt{m}} - \inner{\frac{b_{\ell+1}^{(1)}}{\sqrt{m}}}{\frac{b_{\ell+1}^{(1)}}{\sqrt{m}}}\frac{2}{m}\Tr(D_\ell^{(1)}D_\ell^{(2)}) \bigg| \leq (4+ 4\sqrt{2})M \sqrt{\frac{1+ \log\frac{6}{\delta}}{m}},
\end{align*}
where 
\[
M = \max\left\{ \frac{\|b_{\ell+1}\|^2}{m}, \frac{\|\tb_{\ell+1}\|^2}{m}\right \}.
\]
\end{lem}
\begin{proof}[proof of Lemma~\ref{lem:projGperp}]
We prove the lemma on any realization of $(A,W_1,V_1,\cdots,W_{\ell-1},V_{\ell-1},W_{\ell},W_{\ell+1},V_{\ell+1},\cdots,W_{L},V_{L},v)$, $V_\ell\phi_{W_\ell}(x_{\ell-1})$ and $V_\ell\phi_{W_\ell}(\tx_{\ell-1})$, and consider the remaining randomness of $V_\ell$. In this case, $D_\ell$, $\tD_\ell$, $b_{\ell+1}$ and $\tb_{\ell+1}$ are fixed.

One can show that conditioned on the realization of $V_\ell G$ (whose ``degree of freedom'' is $2m$), $V_\ell \Pi_G^\perp$ is identically distributed as $\widetilde{V}_\ell \Pi_G^\perp$, where $\widetilde{V}_\ell$ is an i.i.d. copy of $V_\ell$. The remaining $m^2-2m$ ``degree of freedom'' is enough for a good concentration. For the proof of this result, we refer the readers to {Lemma~E.3} in \cite{arora2019exact}. 

Denote $ T = \Pi_G^\perp D_\ell^{(1)} D_\ell^{(2)} \Pi_G^\perp $, 
\[
S = \left[\begin{array}{c}
\widetilde{V}_\ell^\top \frac{b_{\ell+1}^{(1)}}{\sqrt{m}} \\
\widetilde{V}_\ell^\top \frac{b_{\ell+1}^{(2)}}{\sqrt{m}}
\end{array}\right].
\]
We know that $S$ is a $2m$-dimensional Gaussian random vector, and 
\[
S\sim\cN\left(0,\mattwo{ \inner{\frac{b_{\ell+1}^{(1)}}{\sqrt{m}}}{\frac{b_{\ell+1}^{(1)}}{\sqrt{m}}} I_m}{ \inner{\frac{b_{\ell+1}^{(1)}}{\sqrt{m}}}{\frac{b_{\ell+1}^{(2)}}{\sqrt{m}}}I_m }{ \inner{\frac{b_{\ell+1}^{(2)}}{\sqrt{m}}}{\frac{b_{\ell+1}^{(1)}}{\sqrt{m}}} I_m}{ \inner{\frac{b_{\ell+1}^{(2)}}{\sqrt{m}}}{\frac{b_{\ell+1}^{(2)}}{\sqrt{m}}} I_m}\right).
\]
Then there exists a matrix $P \in \RR^{2m\times 2m}$, such that 
\[
PP^\top = \mattwo{ \inner{\frac{b_{\ell+1}^{(1)}}{\sqrt{m}}}{\frac{b_{\ell+1}^{(1)}}{\sqrt{m}}} I_m}{ \inner{\frac{b_{\ell+1}^{(1)}}{\sqrt{m}}}{\frac{b_{\ell+1}^{(2)}}{\sqrt{m}}}I_m }{ \inner{\frac{b_{\ell+1}^{(2)}}{\sqrt{m}}}{\frac{b_{\ell+1}^{(1)}}{\sqrt{m}}} I_m}{ \inner{\frac{b_{\ell+1}^{(2)}}{\sqrt{m}}}{\frac{b_{\ell+1}^{(2)}}{\sqrt{m}}} I_m},
\]
and $S \overset{d}{=} P\xi$, $\xi \sim \cN(0,I_{2m})$.

Thus
\begin{align*}
& \quad \frac{b_{\ell+1}^{(1)}}{\sqrt{m}}^\top \widetilde{V}_\ell \Pi_G^\perp D_\ell^{(1)} D_\ell^{(2)} \Pi_G^\perp \widetilde{V}_\ell^\top \frac{b_{\ell+1}^{(2)}}{\sqrt{m}} \overset{d}{=} \xi^\top P^\top \left[\begin{array}{c}I_m \\ 0\end{array}\right]^\top T \left[\begin{array}{c}0 \\ I_m \end{array}\right] P \xi  = \frac{1}{2} \xi^\top P^\top \left[\begin{array}{cc}0 & T \\T& 0\end{array}\right] P \xi. \\
\end{align*}

We have
\begin{align*}
    \left\| \frac{1}{2} P^\top \left[\begin{array}{cc}0 & T \\T& 0\end{array}\right] P \right\|   &\leq \frac{1}{2} \left\|P^\top\right\| \cdot \left\|P\right\|  \cdot \left\| \left[\begin{array}{cc}0 & T \\T& 0\end{array}\right] \right\| \\
    & = \frac{1}{2} \left\|PP^\top\right\| \cdot \left\| T  \right\| \\
    & \leq \frac{1}{2} \left\| \mattwo{ \inner{\frac{b_{\ell+1}^{(1)}}{\sqrt{m}}}{\frac{b_{\ell+1}^{(1)}}{\sqrt{m}}} I_m}{ \inner{\frac{b_{\ell+1}^{(1)}}{\sqrt{m}}}{\frac{b_{\ell+1}^{(2)}}{\sqrt{m}}}I_m }{ \inner{\frac{b_{\ell+1}^{(2)}}{\sqrt{m}}}{\frac{b_{\ell+1}^{(1)}}{\sqrt{m}}} I_m}{ \inner{\frac{b_{\ell+1}^{(2)}}{\sqrt{m}}}{\frac{b_{\ell+1}^{(2)}}{\sqrt{m}}} I_m}\right\| \left\| \Pi_G^\perp \right\|\left\|D_\ell^{(1)} \right\|\left\|D_\ell^{(2)} \right\|\left\|\Pi_G^\perp \right\|\\
    & \leq \frac{ \inner{\frac{b_{\ell+1}^{(1)}}{\sqrt{m}}}{\frac{b_{\ell+1}^{(1)}}{\sqrt{m}}}+\inner{\frac{b_{\ell+1}^{(2)}}{\sqrt{m}}}{\frac{b_{\ell+1}^{(2)}}{\sqrt{m}}} }{2}  \leq M.
\end{align*}
And $\left\| \frac{1}{2} P^\top \left[\begin{array}{cc}0 & T \\T& 0\end{array}\right] P \right\|_F \leq \sqrt{2m}M$. 

Then by the Hanson-Wright Inequality for Gaussian chaos \citep{boucheron2013concentration}, we have with probability at least $1-\delta/3$,
\begin{align*}
    &\quad \frac{2}{m} \left | \frac{b_{\ell+1}^{(1)}}{\sqrt{m}}^\top \widetilde{V}_\ell \Pi_G^\perp D_\ell^{(1)} D_\ell^{(2)} \Pi_G^\perp \widetilde{V}_\ell^\top \frac{b_{\ell+1}^{(2)}}{\sqrt{m}} - \EE_{\widetilde{V}_{\ell}}\left[ \frac{b_{\ell+1}^{(1)}}{\sqrt{m}}^\top \widetilde{V}_\ell \Pi_G^\perp D_\ell^{(1)} D_\ell^{(2)} \Pi_G^\perp \widetilde{V}_\ell^\top \frac{b_{\ell+1}^{(2)}}{\sqrt{m}} \right] \right |  \\
    &\leq \frac{4}{m} \left( \sqrt{2m}M \sqrt{\log\frac{6}{\delta}} +  M\log\frac{6}{\delta} \right) ,
\end{align*}
Furthermore, we have
\[
    \EE_{\widetilde{V}_{\ell}}\left[ \frac{b_{\ell+1}^{(1)}}{\sqrt{m}}^\top \widetilde{V}_\ell \Pi_G^\perp D_\ell^{(1)} D_\ell^{(2)} \Pi_G^\perp \widetilde{V}_\ell^\top \frac{b_{\ell+1}^{(2)}}{\sqrt{m}} \right] = \inner{\frac{b_{\ell+1}^{(1)}}{\sqrt{m}}}{\frac{b_{\ell+1}^{(1)}}{\sqrt{m}}}\Tr(\Pi_G^\perp D_\ell^{(1)}D_\ell^{(2)}).
\] 
Thus 
\begin{align*}
    & \quad   \bigg| \frac{2}{m}\EE_{\widetilde{V}_{\ell}}\left[ \frac{b_{\ell+1}^{(1)}}{\sqrt{m}}^\top \widetilde{V}_\ell \Pi_G^\perp D_\ell^{(1)} D_\ell^{(2)} \Pi_G^\perp \widetilde{V}_\ell^\top \frac{b_{\ell+1}^{(2)}}{\sqrt{m}} \right] - \inner{\frac{b_{\ell+1}^{(1)}}{\sqrt{m}}}{\frac{b_{\ell+1}^{(1)}}{\sqrt{m}}}\frac{2}{m}\Tr(D_\ell^{(1)}D_\ell^{(2)}) \bigg| \\
    & = \frac{2}{m} \bigg| \inner{\frac{b_{\ell+1}^{(1)}}{\sqrt{m}}}{\frac{b_{\ell+1}^{(1)}}{\sqrt{m}}}\Tr(\Pi_G D_\ell^{(1)}D_\ell^{(2)}) \bigg| \\
    & \leq \frac{2}{m} M \Tr(\Pi_G D_\ell^{(1)}D_\ell^{(2)} \Pi_G) \\
    & \leq \frac{4}{m} M.
\end{align*}
By taking union bound, we have with probability at least $1-\delta$, for all $(x^{(1)}, x^{(2)}) \in \{(x,x),(x,\tx),(\tx,\tx)\}$,
\begin{align*}
    &\quad \bigg|\frac{2}{m} \frac{b_{\ell+1}^{(1)}}{\sqrt{m}}^\top V_\ell \Pi_G^\perp D_\ell^{(1)} D_\ell^{(2)} \Pi_G^\perp V_\ell^\top \frac{b_{\ell+1}^{(2)}}{\sqrt{m}} - \inner{\frac{b_{\ell+1}^{(1)}}{\sqrt{m}}}{\frac{b_{\ell+1}^{(1)}}{\sqrt{m}}}\frac{2}{m}\Tr(D_\ell^{(1)}D_\ell^{(2)}) \bigg| \\
    & \leq  \frac{4}{m} \left( \sqrt{2m}M \sqrt{\log\frac{6}{\delta}} +  M\log\frac{6}{\delta} \right) + \frac{4}{m} M\\
    &\leq (4+ 4\sqrt{2})M \sqrt{\frac{1+ \log\frac{6}{\delta}}{m}} ,
\end{align*}
where the last inequality holds when $m\geq 1+ \log\frac{6}{\delta}$.
\end{proof}

\begin{lem}[Norm controls of $b_{\ell+1}$]\label{lem:nc_b}
Assume the following inequalities hold simultaneously for all $\ell=1,2,\cdots,L$
\[
\Big\|\frac{1}{\sqrt{m}}W_\ell \Big\| \leq C,\quad \Big\|\frac{1}{\sqrt{m}}V_\ell \Big\| \leq C.
\]
Then for any fixed input $x$, $1\leq \ell \leq L$ and $u \in \RR^{m}$, when
\[
    m \geq cL^{2-2\gamma} \log \frac{2L}{\delta'},
\]
with probability at least $1-\delta'$ over the randomness of $W_{\ell+1},V_{\ell+1},\cdots,W_L,V_L, v$, we have
\[
    |\inner{u}{b_{\ell+1}}| \leq C' \|u\| \sqrt{\log \frac{2L}{\delta'}}.
\]
\end{lem}
\begin{proof}[proof of Lemma~\ref{lem:nc_b}]
Denote $u_\ell = u$, and 
\[
    u_{i+1} = \alpha \sqrt{\frac{1}{m}}\sqrt{\frac{2}{m}} V_{i+1}D_{i+1}W_{i+1} u_i + u_i, \quad i=\ell, \ell+1,\cdots, L-1.
\]
One can show that $\inner{u}{b_{\ell+1}} = \inner{v}{u_L}$. Next we show that $\|u_{i+1}\| = (1+\cO(\frac{1}{L}))\|u_i\|$ with high probability. First write
\begin{align*}
    \|u_{i+1}\|^2 &= \|u_i\|^2 + \alpha^2 \bigg \| \sqrt{\frac{1}{m}}\sqrt{\frac{2}{m}} V_{i+1}D_{i+1}W_{i+1} u_i \bigg \|^2+ 2\alpha \left\langle{u_i},{ \sqrt{\frac{1}{m}}\sqrt{\frac{2}{m}} V_{i+1}D_{i+1}W_{i+1} u_i }  \right\rangle.
\end{align*}
By the assumption we have 
\begin{align*}
& \bigg \| \sqrt{\frac{2}{m}}D_{i+1}W_{i+1} u_i \bigg\| \leq \sqrt{2}C \|u_i\|,   \\
& \bigg \| \sqrt{\frac{1}{m}}\sqrt{\frac{2}{m}} V_{i+1}D_{i+1}W_{i+1} u_i \bigg \| \leq  \sqrt{2}C^2 \|u_i\|.   
\end{align*}
With probability at least $1-\delta'/L$ over the randomness of $V_{i+1}$, we have
\[
    \left\|\inner{u_i}{ \sqrt{\frac{1}{m}}\sqrt{\frac{2}{m}} V_{i+1}D_{i+1}W_{i+1} u_i }\right\| \leq  \| u_i\|\cdot \bigg \| \sqrt{\frac{2}{m}}D_{i+1}W_{i+1} u_i \bigg\| \sqrt{\frac{c\log\frac{2L}{\delta'}}{m}}.
\]
Then when
\[
    m \geq cL^{2-2\gamma} \log \frac{2L}{\delta'},
\]
we have 
\begin{align*}
    \|u_{i+1}\|^2 &= \|u_i\|^2 + \alpha^2 \bigg \| \sqrt{\frac{1}{m}}\sqrt{\frac{2}{m}} V_{i+1}D_{i+1}W_{i+1} u_i \bigg \|^2+ 2\alpha \inner{u_i}{ \sqrt{\frac{1}{m}}\sqrt{\frac{2}{m}} V_{i+1}D_{i+1}W_{i+1} u_i } \\
    & \leq (1+2C^4/L)\|u_i\|^2 + 2\alpha \sqrt{2}C\| u_i\|^2 \sqrt{\frac{c\log\frac{2L}{\delta'}}{m}}\\
    & \leq (1 +2C^4/L + 2\sqrt{2}C /L)\|u_i\|^2 = (1+ \cO(1/L))\|u_i\|^2.
\end{align*}
Then with probability at least $1-\delta'(L-1)/L$ we have $\|u_L\| \leq C\|u\|$. Finally the result holds from the standard concentration bound for Gaussian random variables \citep{mohri2018foundations}.
\end{proof}

\subsection{Proofs of Lemma~\ref{lem:key}}
\begin{proof}[proof of Lemma~\ref{lem:key}]
By the assumption, we have
\[
\frac{1}{m}\|b_{\ell+1}\|^2 \leq B_{\ell+1}(x,x) + 1 \leq 4.
\]
Similarly, $\frac{1}{m}\|\tb_{\ell+1}\|^2 \leq 4$. Then by Lemma~\ref{lem:projGperp}, when $m\geq \frac{C}{\epsilon^2}(1+\log\frac{6}{\delta})$, we have for all $(x^{(1)}, x^{(2)}) \in \{(x,x),(x,\tx),(\tx,\tx)\}$,
\[
\bigg|\frac{2}{m} \frac{b_{\ell+1}^{(1)}}{\sqrt{m}}^\top V_\ell \Pi_G^\perp D_\ell^{(1)} D_\ell^{(2)} \Pi_G^\perp V_\ell^\top \frac{b_{\ell+1}^{(2)}}{\sqrt{m}} - \inner{\frac{b_{\ell+1}^{(1)}}{\sqrt{m}}}{\frac{b_{\ell+1}^{(1)}}{\sqrt{m}}}\frac{2}{m}\Tr(D_\ell^{(1)}D_\ell^{(2)}) \bigg| \leq c\epsilon.
\]

Specifically, we have
\[
\bigg \|\sqrt{\frac{2}{m}} \frac{b_{\ell+1}}{\sqrt{m}}^\top V_\ell \Pi_G^\perp D_\ell \bigg\| \leq \sqrt{c\epsilon + \frac{2}{m}\Tr(D_\ell)\frac{1}{m}\|b_{\ell+1}\|^2 } \leq \cO(1),
\]
and similarly
\[
\bigg \|\sqrt{\frac{2}{m}} \frac{\tb_{\ell+1}}{\sqrt{m}}^\top V_\ell \Pi_G^\perp \tD_\ell \bigg\|  \leq \cO(1).
\]
Next we bound 
\[
\bigg \|\frac{b_{\ell+1}}{\sqrt{m}}^\top  V_\ell \Pi_{G}  \bigg\| .
\]
Notice that $\Pi_G$ is a orthogonal projection onto the column space of $G$, which is at most 2-dimension. One can write $\Pi_{G} = u_1u_1^\top +u_2u_2^\top$, where $\|u_i\| = 1$ or $0$. By Lemma~\ref{lem:nc_b}, fixing $u_1, u_2$ and $V_\ell$, w.p greater than $1-\delta'$ over the randomness of $W_{\ell+1},V_{\ell+1},\cdots,W_L,V_L, v$, we have
\[
\bigg|b_{\ell+1}^\top \frac{1}{\sqrt{m}}V_\ell u_i \bigg| \leq C''\sqrt{\log \frac{8L}{\delta'}},
\]
and
\[
\bigg|\tb_{\ell+1}^\top \frac{1}{\sqrt{m}}V_\ell u_i \bigg| \leq C''\sqrt{\log \frac{8L}{\delta'}},
\]
for both $i=1,2$  when
\[
    m \geq cL^{2-2\gamma} \log \frac{8L}{\delta'}.
\]

Therefore
\[
\bigg \|\frac{b_{\ell+1}}{\sqrt{m}}^\top  V_\ell \Pi_{G}  \bigg\|, \bigg \|\frac{\tb_{\ell+1}}{\sqrt{m}}^\top  V_\ell \Pi_{G}  \bigg\| \leq \cO\bigg(\sqrt{\log \frac{8L}{\delta'}}\bigg).
\]
Finally, using $I_m = \Pi_G + \Pi_G^\perp$, we have
\begin{align*}
    &\quad \bigg|\frac{2}{m} \frac{b_{\ell+1}^{(1)}}{\sqrt{m}}^\top V_\ell  D_\ell^{(1)} D_\ell^{(2)} V_\ell^\top \frac{b_{\ell+1}^{(2)}}{\sqrt{m}} - \inner{\frac{b_{\ell+1}^{(1)}}{\sqrt{m}}}{\frac{b_{\ell+1}^{(2)}}{\sqrt{m}}}\frac{2}{m}\Tr(D_\ell^{(1)}D_\ell^{(2)}) \bigg| \\
    &\leq \bigg|\frac{2}{m} \frac{b_{\ell+1}^{(1)}}{\sqrt{m}}^\top V_\ell \Pi_G^\perp D_\ell^{(1)} D_\ell^{(2)} \Pi_G^\perp V_\ell^\top \frac{b_{\ell+1}^{(2)}}{\sqrt{m}} - \inner{\frac{b_{\ell+1}^{(1)}}{\sqrt{m}}}{\frac{b_{\ell+1}^{(1)}}{\sqrt{m}}}\frac{2}{m}\Tr(D_\ell^{(1)}D_\ell^{(2)}) \bigg| \\
    &\qquad+ \sqrt{\frac{2}{m}} \bigg| \frac{b_{\ell+1}^{(1)}}{\sqrt{m}}^\top V_\ell \Pi_G D_\ell^{(1)} D_\ell^{(2)} \Pi_G^\perp V_\ell^\top \frac{b_{\ell+1}^{(2)}}{\sqrt{m}} \sqrt{\frac{2}{m}}\bigg| \\
    &\qquad+\sqrt{\frac{2}{m}} \bigg|\sqrt{\frac{2}{m}} \frac{b_{\ell+1}^{(1)}}{\sqrt{m}}^\top V_\ell \Pi_G^\perp D_\ell^{(1)} D_\ell^{(2)} \Pi_G V_\ell^\top \frac{b_{\ell+1}^{(2)}}{\sqrt{m}} \bigg| \\
    &\qquad + \frac{2}{m} \bigg| \frac{b_{\ell+1}^{(1)}}{\sqrt{m}}^\top V_\ell \Pi_G D_\ell^{(1)} D_\ell^{(2)} \Pi_G V_\ell^\top \frac{b_{\ell+1}^{(2)}}{\sqrt{m}} \bigg| \\
    &\leq c\epsilon + \sqrt{\frac{2}{m}}\cO\bigg(\sqrt{\log \frac{8L}{\delta'}}\bigg) + \frac{2}{m}\cO\bigg({\log \frac{8L}{\delta'}}\bigg) \leq \epsilon.
\end{align*}
The last inequality holds when $m\geq \frac{C}{\epsilon^2}\log\frac{8L}{\delta'}$.
\end{proof}

\subsection{Proof of Lemma~\ref{lem:W}}
\begin{proof}[proof of Lemma~\ref{lem:W}]
The first part of the proof is essentially the same as Lemma~\ref{lem:key}.
Define
\[
    d_{\ell+1} = D_{\ell}\frac{1}{\sqrt{m}}V_{\ell}^\top \frac{b_{\ell+1}}{\sqrt{m}}, \quad \td_{\ell+1} = \tD_{\ell}\frac{1}{\sqrt{m}}V_{\ell}^\top \frac{\tb_{\ell+1}}{\sqrt{m}}.
\]
We know that $d_{\ell+1}$ and $\td_{\ell+1}$ depend on $W_\ell$ only through $W_\ell x_{\ell-1}$ and $W_\ell \tx_{\ell-1}$. Let $H=[ x_{\ell-1},\tx_{\ell-1}]$. 
Then
\begin{align*}
&\quad\Big|\frac{2}{m} \inner{ W_{\ell}^\top d_{\ell+1}^{(1)} }{ W_{\ell}^\top d_{\ell+1}^{(2)}} - 2 \inner{d_{\ell+1}^{(1)}}{ d_{\ell+1}^{(2)}} \Big| \\
&\leq \Big|\frac{2}{m} \inner{ \Pi_H^\perp W_{\ell}^\top d_{\ell+1}^{(1)} }{\Pi_H^\perp W_{\ell}^\top d_{\ell+1}^{(2)}} - 2 \inner{d_{\ell+1}^{(1)}}{ d_{\ell+1}^{(2)}} \Big| +\Big|\frac{2}{m} \inner{ \Pi_H W_{\ell}^\top d_{\ell+1}^{(1)} }{\Pi_H^\perp W_{\ell}^\top d_{\ell+1}^{(2)}}\Big|\\
&\qquad+\Big|\frac{2}{m} \inner{ \Pi_H^\perp W_{\ell}^\top d_{\ell+1}^{(1)} }{\Pi_H W_{\ell}^\top d_{\ell+1}^{(2)}}\Big|+\Big|\frac{2}{m} \inner{ \Pi_H W_{\ell}^\top d_{\ell+1}^{(1)} }{\Pi_H W_{\ell}^\top d_{\ell+1}^{(2)}}\Big|.
\end{align*}
Since $\|d_{\ell+1}\|,\|\td_{\ell+1}\| = \cO(1)$, similar to Lemma~\ref{lem:projGperp}, when $m \geq 1+\log\frac{6}{\delta}$, w.p at least $1-\delta$ we have
\[
\Big|\frac{2}{m} \inner{ \Pi_H^\perp W_{\ell}^\top d_{\ell+1}^{(1)} }{\Pi_H^\perp W_{\ell}^\top d_{\ell+1}^{(2)}} - 2 \inner{d_{\ell+1}^{(1)}}{ d_{\ell+1}^{(2)}} \Big| \leq \cO\bigg(\sqrt{\frac{1+\log\frac{6}{\delta}}{m}}\bigg),
\]
and
\[
\Big\|\sqrt{\frac{2}{m}} \Pi_H^\perp W_{\ell}^\top d_{\ell+1}^{(i)} \Big\| = \cO(1),\quad i=1,2,
\]
Using the same argument as in the proof of Lemma~\ref{lem:key}, we decompose $\Pi_H$ into two vectors $w_1$ and $w_2$, whose randomness comes from $W_1,V_1,\cdots,W_{\ell-1},V_{\ell-1}$. By writing
\[
    w_i^\top W_{\ell}^\top d_{\ell+1}^{(i)} = \inner{b_{\ell+1}^{(i)}}{\frac{1}{\sqrt{m}}V_\ell D_\ell^{(i)}\frac{1}{\sqrt{m}}W_\ell w_i},
\]
we can also apply Lemma~\ref{lem:nc_b}. Then we conclude that w.p. greater than $1-\delta'$ over the randomness of $v$, we have
\[
\| \Pi_H W_{\ell}^\top d_{\ell+1} \|,\| \Pi_H W_{\ell}^\top \td_{\ell+1} \| = \cO\bigg(\sqrt{\log \frac{8L}{\delta'}}\bigg),
\]
when
\[
    m \geq cL^{2-2\gamma} \log \frac{8L}{\delta'}.
\]
Then exactly the same result of Lemma~\ref{lem:key} holds.

For the second part, notice that
\begin{align*}
   \frac{1}{m}\sqrt{\frac{1}{m}}\sqrt{\frac{2}{m}} \inner{W_{\ell}^\top D_{\ell}^{(1)}V_{\ell}^\top   b_{\ell+1}^{(1)}}{b_{\ell+1}^{(2)}}&= \sqrt{\frac{2}{m}} \inner{W_{\ell}^\top D_{\ell}^{(1)}\sqrt{\frac{1}{m}}V_{\ell}^\top   \frac{b_{\ell+1}^{(1)}}{\sqrt{m}}}{\frac{b_{\ell+1}^{(2)}}{\sqrt{m}}}\\
    & = \sqrt{\frac{2}{m}}  \inner{W_\ell^\top d_{\ell+1}^{(1)}}{\frac{b_{\ell+1}^{(2)}}{\sqrt{m}}} \\
    & = \sqrt{\frac{2}{m}}  \inner{\Pi_H^\perp W_\ell^\top d_{\ell+1}^{(1)}}{\frac{b_{\ell+1}^{(2)}}{\sqrt{m}}} + \sqrt{\frac{2}{m}}  \inner{\Pi_H \frac{1}{\sqrt{m}}W_\ell^\top d_{\ell+1}^{(1)}}{{b_{\ell+1}^{(2)}}}.
\end{align*}
Conditioned on $x_{\ell-1}$, $\tx_{\ell-1}$, $W_\ell x_{\ell-1}$, and $W_\ell\tx_{\ell-1}$, $W_\ell$ is independent of $b_{\ell+1}, \tb_{\ell+1}, d_{\ell+1}$, and $\td_{\ell+1}$. Furthermore, we have $\Pi_H^\perp W_\ell^\top =_d \Pi_H^\perp \widehat{W}_\ell^\top$, where $\widehat{W}_\ell$ is an i.i.d. copy of $W_\ell$. Then for the first term, with probability at least $1-\tdelta/2$, we have for all $(x^{(1)},x^{(2)}) \in \{(x,x),(x,\tx),(\tx,x),(\tx,\tx)\}$,
\[
\bigg|\sqrt{\frac{2}{m}}  \inner{\Pi_H^\perp W_\ell^\top d_{\ell+1}^{(1)}}{\frac{b_{\ell+1}^{(2)}}{\sqrt{m}}}\bigg| \leq \left\|\Pi_H^\perp\frac{b_{\ell+1}^{(2)}}{\sqrt{m}}\right\|\| d_{\ell+1}^{(1)}\|\sqrt{\frac{2c\log\frac{16}{\tdelta}}{m}} \leq \cO\bigg(\sqrt{\frac{\log\frac{16}{\tdelta}}{m}}\bigg).
\]

For the second term, write $\Pi_H = w_1w_1^\top +w_2w_2^\top $, where $\|w_i\|=1$ or $0$. Then by Lemma~\ref{lem:nc_b}, with probability at least $1-{\tdelta}/2$, for all $(x^{(1)},x^{(2)}) \in \{(x,x),(x,\tx),(\tx,x),(\tx,\tx)\}$, when $m \geq cL^{2-2\gamma} \log \frac{16L}{\tdelta}$, we have
\begin{align*}
\bigg|\sqrt{\frac{2}{m}}  \inner{w_iw_i^\top \frac{1}{\sqrt{m}}W_\ell^\top d_{\ell+1}^{(1)}}{{b_{\ell+1}^{(2)}}}    \bigg| &= \bigg|\sqrt{\frac{2}{m}}  w_i^\top \frac{1}{\sqrt{m}}W_\ell^\top d_{\ell+1}^{(1)} \inner{w_i}{{b_{\ell+1}^{(2)}}}    \bigg| \\
&\leq \sqrt{\frac{2}{m}} \|w_i\| \Big\|\frac{1}{\sqrt{m}}W_\ell^\top\Big\| \|d_{\ell+1}^{(1)}\| \Big|\inner{w_i}{{b_{\ell+1}^{(2)}}} \Big|\\
&\leq \cO\left(\sqrt{\frac{\log \frac{16L}{\tdelta}}{m}} \right).
\end{align*}

\end{proof}


\section{Proof of Theorem \ref{thm:ffntkinf}}
\label{proof:limitingffntk}

\begin{proof}
For $x, \tx\in \SSS^{D-1}$, we have $K_\ell (x,x) = K_\ell (\tx,\tx) = 1$ for all $\ell$. Hence we only need to study when $x\ne\tx$. Note we have
\begin{align*}
    K_\ell (x,\tx)=\Gamma_\sigma (K_{\ell -1})(x,\tx) = \hat{\sigma}(K_{\ell -1}(x,\tx)),\  \mathrm{and} \ \Gamma_{\sigma'} (K_\ell )(x,\tx) = \widehat{\sigma'}(K_\ell (x,\tx)).
\end{align*}
For simplicity, we use $K_\ell $ to denote $K_\ell (x,\tx)$, where $x \ne \tx $ and $x, \tx \in \SSS^{D-1}$.

Recall that
\begin{align*}
    \hat{\sigma}(\rho) = \frac{\sqrt{1-\rho^{2}}+\left(\pi-\cos ^{-1}(\rho)\right) \rho}{\pi},\ \mathrm{and}\ \widehat{\sigma'}(\rho) = \frac{\pi-\cos ^{-1}(\rho)}{\pi}.
\end{align*}
Hence we have $\hat{\sigma}(1)=1$, $K_{\ell -1}\leq\hat{\sigma} (K_{\ell -1})=K_\ell $,  ${(\hat{\sigma})}'(\rho)=\widehat{\sigma'}(\rho)\in [0,1]$, and ${(\widehat{\sigma'})}'(\rho)\geq0$. Then $\hat{\sigma}$ is a convex function.

Since $\{K_\ell \}$ is an increasing sequence and $|K_\ell |\leq1$, we have $K_\ell $ converges as $\ell\to \infty$. Taking the limit of both sides of $\hat{\sigma} (K_{\ell -1})=K_\ell $, we have $K_\ell \to1$ as $\ell\to\infty$.\\
For $K_\ell $, we also have
\begin{align*}
K_\ell =\hat{\sigma} (K_{\ell -1})& =\frac{\sqrt{1-K_{\ell -1}^2}+(\pi-\cos^{-1}(K_{\ell -1}))K_{\ell -1}}{\pi} =K_{\ell -1}+\frac{\sqrt{1-K_{\ell -1}^2}-\cos^{-1}(K_{\ell -1})K_{\ell -1}}{\pi}.
\end{align*}

Let $e_\ell =1-K_\ell $, we can easily check that 
\begin{align}
\label{xl}
e_{\ell -1}-\frac{e_{\ell -1}^{3/2}}{\pi}\leq e_\ell  \leq e_{\ell -1}-\frac{2\sqrt{2}e_{\ell -1}^{3/2}}{3\pi}.
\end{align}
Hence as $e_\ell  \to 0$, we have $\frac{e_\ell }{e_{\ell -1}} \to 1$, which implies $\{K_\ell \}$ converges sublinearly.

Assume $e_\ell =\frac{C}{\ell ^p}+\cO(\ell^{-(p+1)})$. By taking the assumption into (\ref{xl}) and comparing the highest order of both sides, we have $p=2$. 

Thus $\exists C$, s.t. $|1-K_\ell |\leq \frac{C}{\ell ^2}$, i.e. the convergence rate of $K_\ell $ is $\mathcal{O}\left(\frac{1}{\ell ^2}\right)$.

\begin{lem}
For each $K_0<1$, there exists $p>0$ and $n_0=n_0(\delta)>0$, such that $K_n\leq 1-\frac{9\pi^2}{2 (n+n_0)^{2+\frac{\log(L)^p}{L}}}$, $\forall n=0,\dots,L$, when $L$ is large.
\end{lem}

\begin{proof}

First, solve $K_0\leq1-\frac{9\pi^2}{2 n^{2+\frac{\log(L)^p}{L}}}$. Then we can choose $n_0\geq\sqrt{\frac{9\pi^2}{2\delta}}\geq \sqrt{\frac{9\pi^2}{2(1-K_0)}}$, which is independent of $L$ and $n$. For the rest of the proof, without loss of generality, we just use $n$ instead of $n+n_0$. Also for small $\delta$( when $\delta$ is not small enough we can pick a small $\delta_0<\delta$ and let $n_0\geq\sqrt{\frac{9\pi^2}{2\delta_0}}$), we have $\frac{9\pi^2}{2 (n+n_0)^{2+\frac{\log(L)^p}{L}}}\leq\delta(\mathrm{or}\ \delta_0)$ which is also small.

Let $K_n=1-\epsilon$. Then, when $\epsilon$ is small, we have
\begin{align*}
K_{n+1}-K_n=\hat{\sigma}(K_n)-K_n= \mathcal{O}(\epsilon^{3/2}).
\end{align*}
Also, we have
\begin{align*}
\left(1-\frac{9\pi^2}{2 (n+1)^{2+\frac{\log(L)^p}{L}}}\right)-\left(1-\frac{9\pi^2}{2 n^{2+\frac{\log(L)^p}{L}}}\right)=\mathcal{O}\left(\frac{1}{n^{3+\frac{\log(L)^p}{L}}}\right)\\
\geq \mathcal{O}\left(\left(\frac{1}{n^{2+\frac{\log(L)^p}{L}}}\right)^{3/2}\right)=\mathcal{O}\left(\frac{1}{n^{3+\frac{3\log(L)^p}{2L}}}\right).
\end{align*} 
Overall, we want an upper bound for $K_n$ and from the above we only know that $K_n$ is of order $1-\cO(n^{-2})$ but this order may hide some terms of logarithmic order. Hence we use the order $1-\cO(n^{-(2+\epsilon)})$ to provide an upper bound of $K_n$. Here $\frac{\log(L)^p}{L}$ is constructed for the convenience of the rest of the proof.
\end{proof}

Let $N_0=N_0(L)$ be the solution of 
\begin{align*}
\cos \left(\pi  \left(1-\left(\frac{n+1}{n+2}\right)^{3-\frac{\log(L)^2}{L}}\right)\right)=\hat{\sigma}\left(\cos \left(\pi  \left(1-\left(\frac{n}{n+1}\right)^{3-\frac{\log(L)^2}{L}}\right)\right)\right), 
\end{align*}
where for $N_0<n<N_L$ with some $N_L,$ we have $$\cos \left(\pi  \left(1-\left(\frac{n+1}{n+2}\right)^{3-\frac{\log(L)^2}{L}}\right)\right)\geq\hat{\sigma}\left(\cos \left(\pi  \left(1-\left(\frac{n}{n+1}\right)^{3-\frac{\log(L)^2}{L}}\right)\right)\right).$$  One can check by series expansion that $N_0=N_0(L)\leq 5\frac{L}{\log(L)^2}$. 

Next we would like to find $n$ such that 
\begin{align*}
K_n=\cos \left(\pi  \left(1-\left(\frac{5\frac{L}{\log(L)^2}}{5\frac{L}{\log(L)^2}+1}\right)^{3-\frac{\log(L)^2}{L}}\right)\right). 
\end{align*}
By series expansion, we know 
\begin{align*}
\cos \left(\pi  \left(1-\left(\frac{5\frac{L}{\log(L)^2}}{5\frac{L}{\log(L)^2}+1}\right)^{3-\frac{\log(L)^2}{L}}\right)\right)\geq 1-\frac{9\pi^2}{2 \left(\frac{5L}{\log(L)^2}\right)^2}. 
\end{align*}
Then it suffices to solve 
\begin{align}
\label{findn}
1-\frac{9\pi^2}{2 (\frac{5L}{\log(L)^2})^2}\geq 1-\frac{9\pi^2}{2 n^{2+\frac{\log(L)^p}{L}}} \geq K_n,\ i.e.,\ n^{2+\frac{\log(L)^p}{L}}\leq \left(\frac{5L}{\log(L)^2}\right)^2. 
\end{align}

\begin{lem}
When $q>p-1$, we have $n\lesssim \frac{5L}{\log(L)^2}-\log(L)^q$ satisfies (\ref{findn}). 
\end{lem}

\begin{proof}
If the condition above holds, we have 
\begin{align*}
n^{2+\frac{\log(L)^p}{L}}\leq \left(\frac{5L}{\log(L)^2}-\log(L)^q\right)^{2+\frac{\log(L)^p}{L}}, 
\end{align*}
which is
\begin{align*}
n^{1+\frac{\log(L)^p}{2L}}\leq & \left(\frac{5L}{\log(L)^2}-\log(L)^q\right) \left(\frac{5L}{\log(L)^2}-\log(L)^q\right)^{\frac{\log(L)^p}{2L}} \\ 
\leq &\left(\frac{5L}{\log(L)^2}-\log(L)^q\right)\left(1+\frac{\log(L)^p\log(\frac{5L}{\log(L)^2})}{2L}\right)\\ 
= &\frac{5L}{\log(L)^2}-\log(L)^q+\frac{5}{2}\log(L)^{p-2}\log\left(\frac{5L}{\log(L)^2}\right)-\frac{1}{2L}\log(L)^{p+q}\log\left(\frac{5L}{\log(L)^2}\right),
\end{align*}
where $\left(\frac{5L}{\log(L)^2}-\log(L)^q\right)^{\frac{\log(L)^p}{2L}}\to1$ as $L\to\infty$.

Thus we have $q>p-1$. 
\end{proof}
Just pick $q=p$. Then we have $n^{1+\frac{\log(L)^p}{2L}}\lesssim \frac{5L}{\log(L)^2}$ and $n\lesssim \frac{5L}{\log(L)^2}-\log(L)^p$.

\begin{lem}
\label{boundsfork}
When $L$ is large enough, we have
\begin{align*}
\cos \left(\pi  \left(1-\left(\frac{n}{n+1}\right)^{3+\frac{\log(L)^2}{L}}\right)\right)\leq K_n \leq \cos \left(\pi  \left(1-\left(\frac{n+\log(L)^p}{n+\log(L)^p+1}\right)^{3-\frac{\log(L)^2}{L}}\right)\right).
\end{align*}
\end{lem}
\begin{proof}
Let $F(n)=\cos \left(\pi  \left(1-\left(\frac{n+\log(L)^p}{n+\log(L)^p+1}\right)^{3-\frac{\log(L)^p}{L}}\right)\right)$. 

For the right hand side, when $n\gtrsim \frac{5L}{\log(L)^2}-\log(L)^p$, we have, by series expansion, $F(n+1) \geq \hat{\sigma}\left( F(n)\right)$. Also, when $n\sim a L$, where $0<a\leq1$, we have $$F(n+1)-\hat{\sigma}(F(n))= \mathcal{O}\left(\frac{3 \left(2 \pi ^2 \text{a} \log ^{10}(L)+\pi ^2 \log ^8(L)\right)}{2 L^4 \left(\text{a} \log ^2(L)+5\right)^4}\right)>0.$$ Then for $\frac{5L}{\log(L)^2}-\log(L)^p\lesssim n\lesssim L$, we have $F(n+1)\geq\hat{\sigma}\left( F(n)\right)$ and thus $K_n\leq F(n)$. 

When $n\lesssim \frac{5L}{\log(L)^2}-\log(L)^p$, we have $F(n+1) \leq \hat{\sigma}\left( F(n)\right)$. Hence $K_n \leq F(n)$.

For the left hand side, 
\begin{align*}
\cos \left(\pi  \left(1-\left(\frac{n+1}{n+2}\right)^{3+\frac{\log(L)^2}{L}}\right)\right)-\hat{\sigma}\left(\cos \left(\pi  \left(1-\left(\frac{n}{n+1}\right)^{3+\frac{\log(L)^2}{L}}\right)\right)\right)\\ 
\sim -\frac{27\pi^2}{2n^4}-\frac{3\pi^2\log(L)^2}{n^3L},\ \forall n=1,...,L. 
\end{align*}
Hence we have the left hand side.
\end{proof}

From Lemma~\ref{boundsfork}, by series expansion, we have 
\begin{align*}
|1-K_n|\leq \frac{\left(3\pi+\frac{\pi\log(L)^2}{L}\right)^2}{2n^2}\sim \frac{9\pi^2}{2n^2},
\end{align*}
when $L$ is large.

Moreover, we can get 
\begin{align*}
\left(\frac{n}{n+1}\right)^{3+\frac{\log(L)^2}{L}}\leq \Gamma_{\sigma'}(K_n)\leq \left(\frac{n+\log(L)^p}{n+\log(L)^p+1}\right)^{3-\frac{\log(L)^2}{L}}.
\end{align*}
Then 
\begin{align*}
\left(\frac{\ell -1}{L}\right)^{3+\frac{\log(L)^2}{L}}\leq\prod_{i=\ell }^{L}\Gamma_{\sigma'}(K_{i-1})\leq\left(\frac{\ell +\log(L)^p-1}{L+\log(L)^p}\right)^{3-\frac{\log(L)^2}{L}}.
\end{align*}

Let $N=\log(L)^p$. For the right hand side, if we sum over $\ell$, we have
\begin{align*}
\frac{1}{L}\sum_{\ell =1}^{L} \left(\frac{\ell +N-1}{L+N}\right)^{3-\frac{\log(L)^2}{L}}  &\leq\frac{1}{L}\int_1^{L+1} \left(\frac{x+N-1}{L+N}\right)^{3-\frac{\log(L)^2}{L}}dx\\
& =\frac{\left(\left(L+N\right)^{4-\frac{\log(L)^2}{L}}-\left(N\right)^{4-\frac{\log(L)^2}{L}}\right)}{L(L+N)^{3-\frac{\log(L)^2}{L}}\left(4-\frac{\log(L)^2}{L}\right)}.
\end{align*}

Taking the limit of both sides, we have
\begin{align*}
\lim_{L\to\infty}\frac{1}{L}\sum_{\ell =1}^L \left(\frac{\ell +N-1}{L+N}\right)^{3-\frac{\log(L)^2}{L}}\leq \frac{1}{4}.
\end{align*} 

Similarly, by 
\begin{align*}
\frac{1}{L}\sum_{i=1}^L \left(\frac{\ell -1}{L}\right)^{3+\frac{\log(L)^2}{L}}\geq\frac{1}{L}\int_1^{L} \left(\frac{x-1}{L}\right)^{3+\frac{\log(L)^2}{L}}dx=\frac{\left(L-1\right)^{4+\frac{\log(L)^2}{L}}}{\left(4+\frac{\log(L)^2}{L}\right)L^{4+\frac{\log(L)^2}{L}}},
\end{align*}
we have 
\begin{align*}
\lim_{L\to\infty}\frac{1}{L}\sum_{i=1}^L \left(\frac{\ell -1}{L}\right)^{3 + \frac{\log(L)^2}{L}}\geq \frac{1}{4}.
\end{align*} 

Hence, 
\begin{align*}
\lim_{L\to\infty}\frac{1}{L}\sum_{\ell =1}^{L}& \left(\frac{\ell +N-1}{L+N}\right)^{3-\frac{\log(L)^2}{L}}=\lim_{L\to\infty}\frac{1}{L}\sum_{\ell=1}^L \left(\frac{\ell -1}{L}\right)^{3+\frac{\log(L)^2}{L}}\\
 &= \lim_{L\to\infty}\frac{1}{L}\sum_{\ell=1}^L\prod_{i=\ell }^{L}\Gamma_{\sigma'}(K_{i-1})=\frac{1}{4}.
\end{align*}

Recall from previous discussion, $K_\ell =1-\mathcal{O}(\frac{1}{\ell ^2})$. Therefore,
\begin{align*}
\lim_{L\to\infty}\frac{1}{L}\sum_{\ell=1}^LK_{\ell -1}\prod_{i=\ell }^{L}\Gamma_{\sigma'}(K_{i-1})=\frac{1}{4}.
\end{align*}

Also, when $L$ is large, we have 
\begin{align*}
\frac{\left(\left(L+N\right)^{4-\frac{\log(L)^2}{L}}-\left(N\right)^{4-\frac{\log(L)^2}{L}}\right)}{L(L+N)^{3-\frac{\log(L)^2}{L}}\left(4-\frac{\log(L)^2}{L}\right)}>\frac{1}{4}>\frac{(L-1)^{4+\frac{\log(L)^2}{L}}}{\left(4+\frac{\log(L)^2}{L}\right)L^{4+\frac{\log(L)^2}{L}}}.
\end{align*}

Hence we can estimate the convergence rate of the normalized kernel
\begingroup
\allowdisplaybreaks
\begin{align*}
\bigg|\frac{1}{L}\sum_{\ell =1}^LK_{\ell -1}\prod_{i=\ell }^{L}\Gamma_{\sigma'}(K_{i-1})-\frac{1}{4}\bigg|&=\bigg|\frac{1}{L}\sum_{\ell =1}^L\left(K_{\ell -1}\left(\prod_{i=\ell }^{L}\Gamma_{\sigma'}(K_{i-1})-\frac{1}{4}\right)+\frac{1}{4}(K_{\ell -1}-1)\right)\bigg|\\
&\leq \bigg|\frac{1}{L}\sum_{\ell =1}^L\prod_{i=\ell }^{L}\Gamma_{\sigma'}(K_{i-1})-\frac{1}{4}\bigg|+\frac{1}{4}\bigg|\frac{1}{L}\sum_{\ell =1}^{L}(K_{\ell -1}-1)\bigg|\\
&\leq \left|\frac{\left(\left(L+N\right)^{4-\frac{\log(L)^2}{L}}-\left(N\right)^{4-\frac{\log(L)^2}{L}}\right)}{L(L+N)^{3-\frac{\log(L)^2}{L}}\left(4-\frac{\log(L)^2}{L}\right)}-\frac{(L-1)^{4+\frac{\log(L)^2}{L}}}{\left(4+\frac{\log(L)^2}{L}\right)L^{4+\frac{\log(L)^2}{L}}}\right|\\
& \qquad\qquad +\frac{1}{4}\bigg|\frac{1}{L}\sum_{i=1}^{L}(K_{\ell -1}-1)\bigg|\\
& \lesssim \frac{4\log(L)^p+\log(L)^2}{16L}=\mathcal{O}\left(\frac{\mathrm{poly} \log(L))}{L}\right)
\end{align*}
\endgroup
\end{proof}


\section{Proof of Theorem \ref{thm:resntkinf}}
\label{proof:limitingrnntk}

\begin{proof}
We denote $K_{\ell ,L}$ to be the $\ell$-th layer of $K$ when the depth is $L$, which is originally denoted by $K_\ell $.

Let $S_{\ell ,L}=\frac{K_{\ell ,L}}{(1+\alpha^2)^\ell}=\frac{K_{\ell ,L}}{(1+1/L^2)^\ell}$ and $S_0=K_0$, then $\Gamma_\sigma(K_{\ell ,L})=(1+\alpha^2)^\ell\hat{\sigma}(S_{\ell ,L})$ and $\Gamma_{\sigma'}(K_{\ell ,L})=\widehat{\sigma'}(S_{\ell ,L})$. Hence we can rewrite the recursion to be
\begin{equation}
S_{\ell ,L}=\frac{S_{\ell -1,L}+\alpha^2\hat{\sigma}(S_{\ell -1,L})}{(1+\alpha^2)}\geq S_{\ell -1,L}.
\end{equation}

Moreover, since $S_{\ell,L}-S_{\ell-1,L}=\frac{\alpha^2}{1+\alpha^2}(\hat{\sigma}(S_{\ell -1,L})-S_{\ell-1,L})$ and $(\hat{\sigma}(S_{\ell -1,L})-S_{\ell-1,L})$ is decreasing, we can have $$S_{\ell ,L}\leq S_0+\frac{(\hat{\sigma}(S_0)-S_0)\ell}{L^2}.$$

Denote $P_{\ell +1,L}=B_{\ell +1,L}(1+\alpha^2)^{-(L-\ell )}=\prod_{i=\ell }^{L-1}\frac{1+\alpha^2\widehat{\sigma'}(S_{i,L})}{1+\alpha^2}$. 
Since 
\begin{align*}
1-\frac{1+\alpha^2\widehat{\sigma'}(S_{i,L})}{1+\alpha^2}=\frac{\alpha^2(1-\widehat{\sigma'}(S_{i,L}))}{1+\alpha^2}=\frac{1-\widehat{\sigma'}(S_{i,L})}{L^2+1},
\end{align*}
we have 
\begin{align*}
1-P_{\ell +1,L}=1-\prod_{i=\ell }^{L-1} \bigg(1-\frac{1-\widehat{\sigma'}(S_{i,L})}{L^2+1}\bigg)\leq \sum_{i=\ell }^{L-1}\frac{1-\widehat{\sigma'}(S_{i,L})}{L^2+1}=\frac{L-\ell -\sum_{i=\ell }^{L-1}\widehat{\sigma'}(S_{i,L})}{L^2+1},
\end{align*}
where $ \ell=1,\dots,L-1.$
For $P_{L+1,L}$, we have $1-P_{L+1,L}=0$.

Then we can rewrite the normalized kernel to be 
\begin{align*}
\overline{\Omega}_L=\frac{1}{2L}\sum_{\ell=1}^L P_{\ell +1,L}(\hat{\sigma}(S_{\ell -1,L})+S_{\ell -1,L}\widehat{\sigma'}(S_{\ell -1,L})).
\end{align*}

Hence we have the bound for each layer 
\begingroup
\allowdisplaybreaks
\begin{align*}
\Big| P_{\ell +1,L}&(\hat{\sigma}(S_{\ell -1,L})+S_{\ell -1,L}\widehat{\sigma'}(S_{\ell -1,L}))-(\hat{\sigma}(S_{0})+S_{0}\widehat{\sigma'}(S_{0}))\Big|\\
&\leq \Big|P_{\ell +1,L}\Big|\cdot \Big|(\hat{\sigma}(S_{\ell -1,L})+S_{\ell -1,L}\widehat{\sigma'}(S_{\ell -1,L}))-(\hat{\sigma}(S_{0})+S_{0}\widehat{\sigma'}(S_{0}))\Big|+\Big|\hat{\sigma}(S_{0})+S_{0}\widehat{\sigma'}(S_{0})\Big|\cdot\Big|1-P_{\ell +1,L}\Big|\\
& \leq \Big|\widehat{\sigma'}(S_{\ell -1,L})(S_{\ell -1,L}-S_0)\Big|+\Big|\widehat{\sigma'}(S_{\ell -1,L})S_{\ell -1,L}-\widehat{\sigma'}(S_{0})S_0\Big|+\Big|\hat{\sigma}(S_{0})+S_{0}\widehat{\sigma'}(S_{0})\Big|\cdot\Big|1-P_{\ell +1,L}\Big|\\
& = 2\Big|\widehat{\sigma'}(S_{\ell -1,L})(S_{\ell -1,L}-S_0)\Big|+\Big|S_0(\widehat{\sigma'}(S_{\ell -1,L})-\widehat{\sigma'}(S_{0}))\Big|+\Big|\hat{\sigma}(S_{0})+S_{0}\widehat{\sigma'}(S_{0})\Big|\cdot\Big|1-P_{\ell +1,L}\Big|\\
& \leq \frac{2\widehat{\sigma'}(S_{\ell -1,L})(\hat{\sigma}(S_0)-S_0)\ell}{L^2}+\frac{|S_0|(\hat{\sigma}(S_0)-S_0)(\ell-1)}{\pi L^2\sqrt{1-S_{\ell -1,L}^2}}+\Big|\hat{\sigma}(S_{0})+S_{0}\widehat{\sigma'}(S_{0})\Big|\frac{L-\ell -\sum_{i=\ell }^{L-1}\widehat{\sigma'}(S_{i,L})}{L^2+1}\\
& \leq \frac{2\widehat{\sigma'}(S_{\ell -1,L})(\hat{\sigma}(S_0)-S_0)\ell}{L^2}+\frac{|S_0|(\hat{\sigma}(S_0)-S_0)(\ell-1)}{\pi L^2\sqrt{1-S_{\ell -1,L}^2}}+\Big|\hat{\sigma}(S_{0})+S_{0}\widehat{\sigma'}(S_{0})\Big|\frac{L-\ell -(L-\ell )\widehat{\sigma'}(S_0)}{L^2+1}.
\end{align*}
\endgroup

Therefore we have the bound for the normalized kernel
\begingroup
\allowdisplaybreaks
\begin{align*}
\bigg|\overline{\Omega}_L&-\frac{1}{2}(\hat{\sigma}(S_{0})+S_{0}\widehat{\sigma'}(S_{0}))\bigg|\\
&=\bigg|\frac{1}{2L}\sum_{\ell=1}^L\left( P_{\ell +1,L}(\hat{\sigma}(S_{\ell -1,L})+S_{\ell -1,L}\widehat{\sigma'}(S_{\ell -1,L}))\right)-\frac{1}{2}(\hat{\sigma}(S_{0})+S_{0}\widehat{\sigma'}(S_{0}))\bigg|\\
& \leq \frac{1}{2L}\sum_{\ell=1}^L\left(\frac{2\widehat{\sigma'}(S_{\ell -1,L})(\hat{\sigma}(S_0)-S_0)\ell}{L^2}+\frac{|S_0|(\hat{\sigma}(S_0)-S_0)(\ell-1)}{\pi L^2\sqrt{1-S_{\ell -1,L}^2}}\right)\\
& \qquad\qquad+ \frac{1}{2L}\sum_{\ell=1}^{L-1}\left(\Big|\hat{\sigma}(S_{0})+S_{0}\widehat{\sigma'}(S_{0})\Big|\frac{L-\ell -(L-\ell )\widehat{\sigma'}(S_0)}{L^2+1}\right)\\
&\leq \frac{1}{2L}\left(\frac{L+1}{L}(\hat{\sigma}(S_0)-S_0)+\frac{|S_0|(\hat{\sigma}(S_0)-S_0)L(L-1)}{2\pi L^2 C}+\Big|\hat{\sigma}(S_{0})+S_{0}\widehat{\sigma'}(S_{0})\Big|\frac{\frac{L(L-1)}{2}(1-\widehat{\sigma'}(S_0))}{L^2+1}\right)\\
&\sim\left(\frac{(\hat{\sigma}(S_0)-S_0)}{2}\left(1+\frac{|S_0|}{2\pi C}\right)+\frac{1}{2}\Big|\hat{\sigma}(S_{0})+S_{0}\widehat{\sigma'}(S_{0})\Big|(1-\widehat{\sigma'}(S_0)) \right)\frac{1}{L}
\end{align*}
\endgroup
where $C=C(\delta)=\sqrt{1-(1-\delta)^2}$ and $S_0=K_0$.
\end{proof}

\end{document}